\documentclass{article}[11pt]
\usepackage{fullpage}
\usepackage{indentfirst}
\usepackage[utf8]{inputenc} 
\usepackage[T1]{fontenc}    
\usepackage{url}            
\usepackage{booktabs}       
\usepackage{amsfonts}       
\usepackage{nicefrac}       
\usepackage{microtype}      
\usepackage[ruled,vlined,linesnumbered]{algorithm2e}
\usepackage{amsmath}
\usepackage{amsthm}
\usepackage{verbatim}
\usepackage{enumerate}
\usepackage{color}
\usepackage[page,title,titletoc,header]{appendix}
\usepackage{tikz}
\usetikzlibrary{arrows,automata}

\usepackage{tikz,tikz-qtree}
\usepackage{colortbl}
\usepackage{multirow}
\usepackage{array}
\usepackage{wrapfig}
\newcommand{\lijie}[1]{{\color{red} Lijie: #1}}
\newcommand{\jian}[1]{{\color{blue} Jian: #1}}
\newcommand{\mingda}[1]{{\color{purple} Mingda: #1}}

\newcommand{\algo}{\textsf{Bilateral-Elimination}}
\newcommand{\arm}{A}
\newcommand{\armp}{\arm'}
\newcommand{\bestarm}{Best-$1$-Arm}
\newcommand{\bestkarms}{Best-$k$-Arm}
\newcommand{\CORRECT}{$\delta$-correct}
\newcommand{\eat}[1]{{}}
\newcommand{\eps}{\varepsilon}
\newcommand{\event}{\mathcal{E}}
\newcommand{\Ex}{\mathrm{E}}
\newcommand{\Gap}[1]{\Delta_{[#1]}}
\newcommand{\badevent}{\event^\textsf{bad}}
\newcommand{\goodevent}{\event^\textsf{good}}
\newcommand{\validevent}{\event^\textsf{valid}}
\newcommand{\KL}{\mathrm{KL}}
\newcommand{\marm}{M}
\newcommand{\Mean}[1]{\mu_{[#1]}}
\newcommand{\normal}{\mathcal{N}}
\newcommand{\argmax}{\textrm{argmax}}

\newcommand{\alg}{\mathbb{A}}
\newcommand{\algemb}{\alg^{\textsf{emb}}}

\newcommand{\algnew}{\alg^{\textsf{new}}}
\newcommand{\algsym}{\alg^{\textsf{sym}}}
\newcommand{\inst}{\mathcal{I}}
\newcommand{\instden}{\inst^{\textsf{dense}}}
\newcommand{\instemb}{\inst^{\textsf{emb}}}
\newcommand{\instmix}{\inst^{\textsf{mix}}}
\newcommand{\instnew}{\inst^{\textsf{new}}}
\newcommand{\instp}{\inst'}
\newcommand{\instsym}{\inst^{\textsf{sym}}}
\newcommand{\exmix}{\textrm{Expr}^\textsf{mix}}
\newcommand{\exsym}{\textrm{Expr}^\textsf{sym}}

\newcommand{\gbig}{G^{\textsf{large}}}
\newcommand{\gsma}{G^{\textsf{small}}}
\newcommand{\hbig}{H^{\textsf{large}}}
\newcommand{\hsma}{H^{\textsf{small}}}
\newcommand{\htil}{\widetilde H}
\newcommand{\htilbig}{\widetilde H^{\textsf{large}}}
\newcommand{\htilsma}{\widetilde H^{\textsf{small}}}
\newcommand{\Hone}{H^{(1)}}
\newcommand{\Htwobig}{H^{(2,\textsf{large})}}
\newcommand{\Htwosma}{H^{(2,\textsf{small})}}
\newcommand{\kbig}{k^\textsf{large}}
\newcommand{\ksma}{k^\textsf{small}}
\newcommand{\sbig}{S^\textsf{large}}
\newcommand{\ssma}{S^\textsf{small}}
\newcommand{\thetabig}{\theta^\textsf{large}}
\newcommand{\thetasma}{\theta^\textsf{small}}
\newcommand{\mubig}{\mu^\textsf{large}}
\newcommand{\musma}{\mu^\textsf{small}}
\newcommand{\deltap}{\delta'}
\newcommand{\ubig}{U^\textsf{large}}
\newcommand{\usma}{U^\textsf{small}}
\newcommand{\obs}{\mathrm{obs}}

\newcommand{\pdf}[2]{f_{#1}(#2)}
\newcommand{\pos}{p^*}
\newcommand{\Long}{\mathsf{Long}}
\newcommand{\Short}{\mathsf{Short}}

\newcommand{\range}{R}

\newcommand{\PAC}{\textsf{PAC-Best-k}}
\newcommand{\PACSP}{\textsf{PAC-SamplePrune}}
\newcommand{\ME}{\textsf{EstMean}}
\newcommand{\MEbig}{\textsf{EstMean-Large}}
\newcommand{\MEsma}{\textsf{EstMean-Small}}
\newcommand{\EL}{\textsf{Elim}}
\newcommand{\ELbig}{\textsf{Elim-Large}}
\newcommand{\ELsma}{\textsf{Elim-Small}}

\newtheorem{Lemma}{Lemma}[section]
\newtheorem{Theorem}{Theorem}[section]
\newtheorem{Definition}{Definition}[section]

\newtheorem{Observation}{Observation}[section]

\newtheorem{Fact}{Fact}[section]

\title{Nearly Instance Optimal Sample Complexity Bounds \\ for 
	Top-k Arm Selection}

\author{
	Lijie Chen \qquad\qquad Jian Li \qquad\qquad Mingda Qiao\\
 	Institute for Interdisciplinary Information Sciences (IIIS), Tsinghua University
}

\date{}

\begin{document}

%

%



\maketitle

\begin{abstract}
In the \bestkarms{} problem, we are given $n$ stochastic bandit arms, each associated with an unknown reward distribution. We are required to identify the $k$ arms with the largest means by taking as few samples as possible. 
In this paper, we make progress towards a complete characterization of the instance-wise sample complexity bounds for the \bestkarms{} problem. On the lower bound side, we obtain a novel complexity term to measure the sample complexity that every \bestkarms{} instance requires. 
This is derived by an interesting and nontrivial reduction from the \bestarm{} problem.
We also provide an elimination-based algorithm that matches the instance-wise lower bound within doubly-logarithmic factors. 
The sample complexity of our algorithm strictly dominates the state-of-the-art for \bestkarms{} (module constant factors).
\end{abstract}

\section{INTRODUCTION}

The stochastic multi-armed bandit is a classical and well-studied model for characterizing the exploration-exploitation tradeoff in various decision-making problems in stochastic settings. 
The most well-known objective in the multi-armed bandit model is to maximize the cumulative gain (or equivalently, to minimize the cumulative regret) that the agent achieves.
Another line of research, called the pure exploration multi-armed bandit problem,
which is motivated by a variety of practical applications including medical trials~\cite{robbins1985some,audibert2010best}, communication
network~\cite{audibert2010best}, and crowdsourcing~\cite{zhou2014optimal,cao2015top}, 
has also attracted significant attention recently.
In the pure exploration problem, the agent draws samples from the arms adaptively (the exploration phase), and finally commits to one of the feasible solutions specified by the problem. In a sense, the exploitation phase in the pure exploration problem simply consists of exploiting the solution to which the agent commits indefinitely. Therefore, the agent's objective is to identify the optimal (or near-optimal) feasible solution with high probability.

In this paper, we focus on the problem of identifying the top-$k$ arms 
(i.e., the $k$ arms with the largest means)
in a stochastic multi-armed bandit model.
The problem is known as the \bestkarms{} problem, and has been extensively studied in the past decade~\cite{kalyanakrishnan2010efficient,  gabillon2012best, gabillon2011multi, kalyanakrishnan2012pac, bubeck2012multiple, kaufmann2013information, zhou2014optimal, kaufmann2015complexity,simchowitz2016towards}. 
We formally define the \bestkarms{} problem as follows.

\begin{Definition}[\bestkarms{}]\label{def:bestkarms}
An instance of \bestkarms{} is a set of stochastic arms $\inst = \{A_1,A_2,\ldots,A_n\}$. Each arm has a $1$-sub-Gaussian reward distribution with an unknown mean in $[0,1/2]$. 

At each step, algorithm $\alg$ chooses an arm and observes an i.i.d. sample from its reward distribution. The goal of $\alg$ is to identify the $k$ arms with the largest means in $\inst$ using as few samples as possible. Let $\Mean{i}$ denote the $i$-th largest mean in an instance of \bestkarms{}. We assume that $\Mean{k} > \Mean{k+1}$ in order to ensure the uniqueness of the solution.
\end{Definition}

Note that in our upper bound, we assume that all reward distributions are 
$1$-sub-Gaussian\footnote{
	A distribution $\mathcal{D}$ is $\sigma$-sub-Gaussian, if it holds that
	$\Ex_{X\sim\mathcal{D}}[\exp(tX-t\Ex_{X\sim\mathcal{D}}[X])]\le\exp(\sigma^2t^2/2)$
	for all $t\in\mathbb{R}$.
}, which is a standard assumption in multi-armed bandit literature.
In our lower bound (Theorem~\ref{theo:lb}), however, we assume that all reward distributions are Gaussian with unit variance.\footnote{For arbitrary distributions, one may 
be able to distinguish two distributions with very close means using very few samples. 
It is impossible to establish a nontrivial lower bound in such generality.
\eat{
	\jian{this is somewhat true. one may have a lower bound expressed in KL..
	we can probably keep the footnote and see if we get any complaint...}
	}} 

When we only want to identify the single best arm, we get the following \bestarm{} problem, which is a well-studied special case of \bestkarms{}.
The problem plays an important role in our lower bound for \bestkarms{}.

\begin{Definition}[\bestarm{}]
The \bestarm{} problem is a special case of \bestkarms{} where $k=1$.
\end{Definition}

Generally, we focus on algorithms that solve \bestkarms{} with probability at least $1 - \delta$.

\begin{Definition}[\CORRECT{} Algorithms]
$\alg$ is a \CORRECT{} algorithm for \bestkarms{} if and only if $\alg$ returns the correct answer with probability at least $1 - \delta$ on every \bestkarms{} instance $\inst$.
\end{Definition}

\subsection{Our Results}

Before stating our results on the \bestkarms{} problem, we first define a few useful notations that characterize the hardness of \bestkarms{} instances.

\subsubsection{Notations}
\textbf{Means and gaps.} Let $\mu_{\arm}$ denote the mean of arm $\arm$. $\Mean{i}$ denotes the $i$-th largest mean among all arms in a specific instance. We define the \textit{gap} of arm $\arm$ as
$$\Delta_{\arm} = \begin{cases}
\mu_{\arm} - \Mean{k+1}, & \mu_{\arm} \ge \Mean{k},\\
\Mean{k} - \mu_{\arm}, & \mu_{\arm} \le \Mean{k+1}.
\end{cases}$$
Note that the gap of an arm is the minimum value by which its mean needs to change in order to alter the top $k$ arms. We let $\Gap{i}$ denote the gap of the $i$-th largest arm.

\eat{\textbf{Accuracy levels.} Throught this paper, $\eps_r=2^{-r}$ serves as the accuracy level we use in the $r$-th round of our algorithm, and it is also used to define the arm groups in the following.}

\textbf{Arm groups.} Let $\eps_r$ denote $2^{-r}$. For an instance $\inst$ of \bestkarms{} and positive integer $r$, we define the arm groups as
$$\gbig_r = \{A\in\inst:\mu_{A}\ge\Mean{k},\Delta_{A}\in(\eps_{r+1},\eps_r]\}\text{, and}$$
$$\gsma_r = \{A\in\inst:\mu_{A}\le\Mean{k+1},\Delta_{A}\in(\eps_{r+1},\eps_r]\}\text{.}$$
In other words, $\gbig_r$ and $\gsma_r$ contain the arms with gaps in $(\eps_{r+1},\eps_r]$ among and outside the best $k$ arms, respectively.

Note that since we assume that the mean of each arm is in $[0,1/2]$, the gap of every arm is at most $1/2$. Therefore by definition each arm is contained in one of the arm groups.

We also use the following shorthand notations:
$$\gbig_{\ge r} = \bigcup_{i = r}^{\infty}\gbig_i\text{ and }\gsma_{\ge r} = \bigcup_{i = r}^{\infty}\gsma_i\text{.}$$

\subsubsection{Lower Bound}
In order to state our instance-wise lower bound precisely, we need to 
elaborate what is an instance.
By Definition \ref{def:bestkarms}, a given instance is a {\em set} of arms,
meaning the particular input order of the arms should not matter.
Note that there indeed exists algorithms that take advantage of the input order
and may perform better for some ``lucky'' input orders than the others.\footnote{
	For example, a sorting algorithm can first check if the input sequence $a_1,\ldots,a_n$
	is in increasing order in $O(n)$ time, and then run an $O(n\log n)$ time algorithm.
	This algorithm is particularly fast for a particular input order. 
	}
In order to prove a tighter lower bound, we need to consider all possible input orders
and take the average. 
From technical perspective, we use the following definition of an instance.

\begin{Definition}[Instance]\label{def:instance}
An instance is considered as a random permutation
of a sequence of arms.
Consequently, the sample complexity of an algorithm on an instance should be considered as 
the average of the number of samples over all permutations.
\end{Definition}

In fact, the random permutation is crucial to establishing \textit{instance-wise} lower bounds for \bestkarms{} (i.e., the minimum number of samples that \textit{every} \CORRECT{} algorithm for \bestkarms{} needs to take on an instance). Without the random permutation, the algorithm might
use fewer samples on some ``lucky" permutations than on others, and it is impossible to prove a tight instance-wise lower bound as ours. 
The use of random permutation to define instance-wise lower bounds is also used in computational geometry~\cite{afshani2009instance} and the \bestarm{} problem~\cite{chen2015optimal, chen2016open}.

We say that an instance of \bestkarms{} is Gaussian, if all reward distributions are normal distributions with unit variance.

\begin{Theorem}\label{theo:lb}
There exists a constant $\delta_0 > 0$, such that for any $\delta < \delta_0$, every \CORRECT{} algorithm for \bestkarms{} takes
$$\Omega \left( H\ln\delta^{-1} + \hbig + \hsma \right)$$
samples in expectation on every Gaussian instance.
Here $H = \sum_{i=1}^{n}\Gap{i}^{-2}\text{,}$
$$\hbig = \sum_{i=1}^{\infty}\left|\gbig_i\right|\cdot\max_{j\le i}\eps_j^{-2}\ln\left|\gsma_{\ge j}\right|\text{, and}$$
$$\hsma = \sum_{i=1}^{\infty}\left|\gsma_i\right|\cdot\max_{j\le i}\eps_j^{-2}\ln\left|\gbig_{\ge j}\right|\text{.}$$
\end{Theorem}

We notice that Simchowitz et al.~\cite{simchowitz2016towards}, independently of our work, derived instance-wise lower bounds for \bestkarms{} similar to Theorem \ref{theo:lb}, using a somewhat different method.

\eat{The complexity term $H$ is well-understood in \bestkarms{} literature, while the other two terms are novel. Note that for $i\ge j$, the algorithm needs to distinguish every arm $\arm\in\gbig_i$ from the arms in $\gsma_{\ge j}$. For this purpose, it is necessary to obtain an $O(\eps_j)$ approximation of the means of both $\arm$ and the other $|\gsma_{\ge j}|$ arms simultaneously. Intuitively, the union bound indicates that $\arm$ must be sampled with $O(\eps_j)$ accuracy and a confidence of $1 - O(1 / |\gsma_{\ge j}|)$, which incurs a sample complexity of $\eps_j^{-2}\ln|\gsma_{\ge j}|$. Then $\hbig$ is simply obtained by taking the maximum among all $j$, and sum over all arm groups $\gbig_i$. The $\hsma$ term is symmetric.}

\subsubsection{Upper Bound}

\begin{Theorem}\label{theo:ub}
For all $\delta > 0$, there is a \CORRECT{} algorithm for \bestkarms{} that takes
$$O\left( H\ln\delta^{-1} + \htil + \htilbig + \htilsma \right)$$
samples in expectation on every instance. 
Here
$$\htil = \sum_{i=1}^{n}\Gap{i}^{-2}\ln\ln\Gap{i}^{-1}\text{,}$$
$$\htilbig = \sum_{i=1}^{\infty}\left|\gbig_i\right|\sum_{j=1}^{i}\eps_j^{-2}\ln\left|\gsma_j\right|\text{, and}$$
$$\htilsma = \sum_{i=1}^{\infty}\left|\gsma_i\right|\sum_{j=1}^{i}\eps_j^{-2}\ln\left|\gbig_j\right|\text{.}$$
\end{Theorem}

The following theorem relates the $\htilbig$ and $\htilsma$ terms to $\hbig$ and $\hsma$ in the lower bound.

\begin{Theorem}\label{theo:ubrefine}
For every \bestkarms{} instance,
 the following statements hold:
\begin{enumerate}
\item
$\htilbig + \htilsma = O\left(\left(\hbig+\hsma\right)\ln\ln n\right)\text{.}$
\item
$\htilbig + \htilsma = O\left(H\ln k\right)\text{.}$
\end{enumerate}
\end{Theorem}

Combining Theorems \ref{theo:lb}, \ref{theo:ub} and \ref{theo:ubrefine}(1), 
our algorithm is \textit{instance-wise optimal} within doubly-logarithmic factors (i.e.,
$\ln\ln n, \ln\ln\Gap{i}^{-1}$). In other words, the sample complexity of our algorithm on \textit{every} single instance nearly matches the minimum number of samples that \textit{every} \CORRECT{} algorithm has to take on that instance.
\eat{\jian{the following two sentences look like a proof of thm 3. but they are not.
	consider removing them.}
In particular, the $\htil$ term is upper bounded by $H\cdot\ln\ln\Gap{k}^{-1}$, where $\Omega(H)$ (in fact, $\Omega(H\ln\delta^{-1})$) is an instance-wise lower bound as shown in \cite{chen2014combinatorial}. The $\htilbig + \htilsma$ term is bounded by $\left(\hbig+\hsma\right)\cdot\ln\ln n$ according to Theorem \ref{theo:ubrefine}.}

Theorem \ref{theo:ub} and Theorem \ref{theo:ubrefine}(2) also imply that our algorithm strictly dominates the state-of-the-art algorithm for \bestkarms{} obtained in \cite{chen2016pure}, which achieves a sample complexity of
\begin{equation*}\begin{split}
&O\left(\sum_{i=1}^{n}\Gap{i}^{-2}\left(\ln\delta^{-1} + \ln k + \ln\ln\Gap{i}^{-1}\right)\right)\\
=& O\left(H\ln\delta^{-1} + H\ln k + \htil\right)\text{.}
\end{split}\end{equation*}

In particular, we give a specific example in Appendix A in which the sample complexity achieved by
Theorem \ref{theo:ub}
is significantly better than that obtained in \cite{chen2016pure}. See Table \ref{tab:prev} for more previous upper bounds on the sample complexity of \bestkarms{}.

\begin{table}[t]
\caption{Upper Bounds of \bestkarms{}} \label{tab:prev}
\begin{center}
\begin{tabular}{ll}
{\bf Source}  &{\bf Sample Complexity} \\
\hline \\
\cite{gabillon2012best}							&$O\left(H\ln\delta^{-1} + H\ln H\right)$ \\
\cite{kalyanakrishnan2012pac}					&$O\left(H\ln\delta^{-1} + H\ln H\right)$ \\
\eat{Chen et al.~}\cite{chen2014combinatorial}	&$O\left(H\ln\delta^{-1} + H\ln H\right)$ \\
\eat{Chen et al.~}\cite{chen2016pure}			&$O\left(H\ln\delta^{-1} + \htil + H\ln k\right)$ \\
This paper \eat{(Theorem \ref{theo:ub})}		&$O\left(H\ln\delta^{-1} + \htil + \htilbig + \htilsma\right)$ \\
\end{tabular}
\end{center}
\end{table}

\subsection{Related Work}
\textbf{\bestarm{}.} In the \bestarm{} problem, the algorithm is required to identify the arm with the largest mean. As a special case of \bestkarms{}, the problem has a history dating back to 1954~\cite{bechhofer1954single}. The problem continues to attract significant attention over the past decade~\cite{audibert2010best,even2006action,mannor2004sample,jamieson2014lil,karnin2013almost,chen2015optimal, carpentier2016tight, garivier2016optimal, chen2016towards}.

\textbf{Combinatorial pure exploration.} The combinatorial pure exploration problem, which further generalizes the cardinality constraint in \bestkarms{} (i.e., to choose exactly $k$ arms) to combinatorial constraints (e.g., matroid constraints), was also studied~\cite{chen2014combinatorial,chen2016pure,gabillon2016improved}.

\textbf{PAC learning.} In the PAC learning setting, the algorithm is required to find an approximate solution to the pure exploration problem. The sample complexity of \bestarm{} and \bestkarms{} in PAC setting has been extensively studied. A tight (worst case) bound of $\Theta(n\eps^{-2}\ln\delta^{-1})$ was obtained for the PAC version of the \bestarm{} problem in~\cite{even2002pac, even2006action, mannor2004sample}. The worst case sample complexity of \bestkarms{} in the PAC setting has also been well-studied~\cite{kalyanakrishnan2010efficient,kalyanakrishnan2012pac,zhou2014optimal,cao2015top}.

\section{PRELIMINARIES}
\textbf{Kullback-Leibler divergence.} Let $\KL(P,Q)$ denote the Kullback-Leibler divergence from distribution $Q$ to $P$. The following well-known fact (e.g., a special case of \cite{duchi2007derivations}) states the Kullback-Leibler divergence between two normal distributions with unit variance.

\begin{Fact}\label{fact:KL}
Let $\normal(\mu,\sigma^2)$ denote the normal distribution with mean $\mu$ and variance $\sigma^2$. It holds that
$$\KL(\normal(\mu_1,1),\normal(\mu_2,1)) = \frac{(\mu_1-\mu_2)^2}{2}\text{.}$$
\end{Fact}

\textbf{Binary relative entropy.} Let $$d(x,y) = x\ln(x/y) + (1-x)\ln[(1-x)/(1-y)]$$ be the binary relative entropy function. The monotonicity of $d(\cdot,\cdot)$ is useful to our following analysis.

\begin{Fact}\label{fact:d}
For $0\le y\le y_0\le x_0\le x\le 1$, $d(x,y)\ge d(x_0,y_0)$.
\end{Fact}

\textbf{Probability and expectation.} $\Pr_{\alg,\inst}$ and $\Ex_{\alg,\inst}$ denote the probability and expectation when algorithm $\alg$ runs on instance $\inst$. These notations are useful since we frequently consider the execution of different algorithms on various instances in our proof of the lower bound.

\textbf{Change of Distribution.}
The following ``Change of Distribution'' lemma, developed in 
\cite{kaufmann2015complexity}, is a useful tool to quantify the behavior of an algorithm when the instance is modified.

\begin{Lemma}[Change of Distribution]\label{lem:CoD}
	Suppose algorithm $\alg$ runs on $n$ arms. $\inst = (\arm_1,\arm_2,\ldots,\arm_n)$ and $\instp = (\armp_1,\armp_2,\ldots,\armp_n)$ are two sequences of arms. $\tau_i$ denotes the number of samples taken on $\arm_i$. For any event $\event$ in $\mathcal{F}_\sigma$, where $\sigma$ is an almost-surely finite stopping time with respect to the filtration $\{\mathcal{F}_t\}_{t\ge0}$, it holds that
	$$\sum_{i=1}^{n}\Ex_{\alg,\inst}[\tau_i]\KL(\arm_i,\armp_i)\ge d\left(\Pr_{\alg,\inst}[\event],\Pr_{\alg,\instp}[\event]\right)\text{.}$$
\end{Lemma}

\section{LOWER BOUND}

\eat{
\lijie{Better specify what an instance mean now...It seems a set of arms and it is somewhat confusing...}
\mingda{Fixed.}
}

Throughout our proof of the lower bound, we assume that the reward distributions of all arms are Gaussian distributions with unit variance. 
Moreover, we assume that the number of arms is sufficiently large. This assumption is used only once in the proof of Lemma \ref{lem:genlb}. Note that when there is only a constant number of arms, our lower bound $\Omega(\hbig + \hsma)$ is implied by the $\Omega(H\ln\delta^{-1})$ term. 

\subsection{Instance Embedding}
The following simple lemma is useful in lower bounding the expected number 
of samples taken from an arm in the top-$k$ set, by restricting to a \bestarm{}
instance embedded in the original \bestkarms{} instance.
We postpone its proof to Appendix C.

\begin{Lemma}[Instance Embedding]\label{lem:InstEmb}
Let $\inst$ be a \bestkarms{} instance. Let $\arm$ be an arm among the top $k$ arms, and $\instemb$ be a \bestarm{} instance consisting of $\arm$ and a subset of arms in $\inst$ outside the top $k$ arms. If some algorithm $\alg$ solves $\inst$ with probability $1-\delta$ while taking less than $N$ samples on $\arm$ in expectation, there exists another algorithm $\algemb$ that solves $\instemb$ with probability $1-\delta$ while taking less than $N$ samples on $\arm$ in expectation.
\end{Lemma}


\subsection{Proof of Theorem \ref{theo:lb}}
We show a lower bound on the number of samples required by each arm separately, 
and then the lower bound stated in Theorem \ref{theo:lb} follows from a direct summation. Formally, we have the following lemma.

\begin{Lemma}\label{lem:karmslb}
Let $\inst$ be an instance of \bestkarms{}. There exist universal constants $\delta$ and $c$ such that for all $1\le j\le i$, any \CORRECT{} algorithm for \bestkarms{} takes at least $c\eps_j^{-2}\ln\left|\gsma_{\ge j}\right|$ samples on every arm $\arm\in\gbig_i$. The same holds if we swap $\gbig$ and $\gsma$. 
\end{Lemma}


Before proving Lemma~\ref{lem:karmslb}, we show that
Theorem \ref{theo:lb} follows from Lemma \ref{lem:karmslb} directly.

\begin{proof}[Proof of Theorem \ref{theo:lb}]
Since the $\Omega(H\ln\delta^{-1})$ lower bound has been established in Theorem 2 of \cite{chen2014combinatorial}, it remains to show that the sample complexity is lower bounded by both $\Omega(\hbig)$ and $\Omega(\hsma)$. Let $\alg$ be a \CORRECT{} algorithm for \bestkarms{}. According to Lemma \ref{lem:karmslb}, $\alg$ draws at least $c\cdot\max_{j\le i}\eps_j^{-2}\ln\left|\gsma_{\ge j}\right|$ samples from each arm in $\gbig_{i}$. Therefore $\alg$ draws at least
\begin{equation*}
\sum_{i=1}^{\infty}\left|\gbig_i\right|\cdot c\cdot\max_{j\le i}\eps_j^{-2}\ln\left|\gsma_{\ge j}\right| = \Omega(\hbig)
\end{equation*}
samples in total from the arms in $\gbig$. The $\Omega(\hsma)$ lower bound is analogous.
\end{proof}

\subsection{Reduction to \bestarm{}}

In order to prove Lemma \ref{lem:karmslb}, we construct a \bestarm{} instance consisting of one arm in $\gbig_i$ and all arms in $\gsma_{\ge j}$. By Instance Embedding (Lemma \ref{lem:InstEmb}), to lower bound the number of samples taken on each arm in $\gbig_i$, it suffices to prove that every algorithm for \bestarm{} takes sufficiently many
samples on the best arm. Formally, we would like to show 
the following key technical lemma.

\begin{Lemma}\label{lem:genlb}
Let $\inst$ be an instance of \bestarm{} consisting of one arm with mean $\mu$ and $n$ arms with means on $[\mu - \Delta,\mu)$. There exist universal constants $\delta$ and $c$ (independent of $n$ and $\Delta$) such that for any algorithm $\alg$ that correctly solves $\inst$ with probability $1-\delta$, the expected number of samples drawn from the optimal arm is at least $c\Delta^{-2}\ln n$.
\end{Lemma}

The proof of Lemma~\ref{lem:genlb} is somewhat technical and we 
present it in the next subsection.
Now we prove Lemma \ref{lem:karmslb} from Lemma~\ref{lem:genlb},
by reducing a \bestarm{} instance to an instance of \bestkarms{} using the Instance Embedding technique. Intuitively, if an algorithm $\alg$ solves \bestkarms{} without taking sufficient number of samples from a specific arm, we may extract an instance of \bestarm{} and derive a contradiction to Lemma \ref{lem:genlb}.

\begin{proof}[Proof of Lemma \ref{lem:karmslb}]
Let $\delta_0$ and $c_0$ be the constants in Lemma \ref{lem:genlb}. We claim that Lemma \ref{lem:karmslb} holds for constants $\delta = \delta_0$ and $c = c_0/4$.

Suppose for a contradiction that when \CORRECT{} algorithm $\alg$ runs on \bestkarms{} instance $\inst$, the number of samples drawn from arm $\arm\in\gbig_i$ is less than $c\eps_j^{-2}\ln\left|\gsma_{\ge j}\right|$ for some $j\le i$.

We construct a \bestarm{} instance $\instnew$ consisting of $\arm$ and all arms in $\gsma_{\ge j}$. By Instance Embedding (Lemma \ref{lem:InstEmb}), there exists algorithm $\algnew$ that solves $\instnew$ with probability $1-\delta$, while the number of samples drawn from arm $\arm$ is upper bounded by $c\eps_j^{-2}\ln\left|\gsma_{\ge j}\right|$ in expectation.

However, Lemma \ref{lem:genlb} implies that $\algnew$ must take more than
$$c_0\Delta^{-2}\ln n \ge 4c (\eps_i + \eps_j)^{-2} \ln\left|\gsma_{\ge j}\right| \ge c\eps_j^{-2}\ln\left|\gsma_{\ge j}\right|$$
samples on the optimal arm, which leads to a contradiction. The case that $\gbig$ and $\gsma$ are swapped is analogous.
\end{proof}

\subsection{Reduction to Symmetric \bestarm{}}
In order to prove Lemma \ref{lem:genlb}, we first study a special case that the instance consists of one optimal arm and several sub-optimal arms with equal means (we call it a Symmetric \bestarm{} instance). For the symmetric \bestarm{} instances, we have the following lower bound on the best arm.

\begin{Lemma}\label{lem:symlb}
Let $\inst$ be an instance of \bestarm{} with one arm with mean $\mu$ and $n$ arms with mean $\mu - \Delta$. There exist universal constants $\delta$ and $c$ (independent of $n$ and $\Delta$) such that for any algorithm $\alg$ that correctly solves $\inst$ with probability $1-\delta$, the expected number of samples drawn from the optimal arm is at least $c\Delta^{-2}\ln n$.
\end{Lemma}

\begin{proof}[Proof of Lemma \ref{lem:symlb}]
We claim that the lemma holds for constants $\delta = 0.5$ and $c = 1$.

Recall that $\normal(\mu, \sigma^2)$ denotes the normal distribution with mean $\mu$ and variance $\sigma^2$. Let $\inst$ be the instance consisting of arm $\arm^*$ with mean $\mu$ and $n$ arms with mean $\mu-\Delta$, and $\instnew$ be the instance obtained from $\inst$ by replacing the reward distribution of $\arm^*$ with $\normal(\mu - \Delta,1)$. $\tau$ denotes the number of samples drawn from $\arm^*$.

Let $\event$ be the event that $\alg$ identifies arm $\arm^*$ as the best arm. Recall that $\Pr_{\alg,\inst}$ and $\Ex_{\alg,\inst}$ denote the probability and expectation when algorithm $\alg$ runs on instance $\inst$ respectively. Since $\alg$ solves $\inst$ correctly with probability at least $1-\delta$, we have $\Pr_{\alg,\inst}[\event]\ge 1-\delta$. On the other hand, $\instnew$ consists of $n+1$ completely identical arms. By Definition \ref{def:instance}, $\alg$ takes a random permutation of $\instnew$ as its input. Therefore the probability that $\alg$ returns each arm is the same, and it follows that $\Pr_{\alg,\instnew}[\event]\le 1/(n+1)$.

By Change of Distribution (Lemma \ref{lem:CoD}), we have
\begin{equation*}\begin{split}
    \frac{1}{2}\Ex_{\alg,\inst}[\tau]\Delta^2
  =&\Ex_{\alg,\inst}[\tau]\cdot\KL(\normal(\mu,1),\normal(\mu-\Delta,1))\\
\ge& d\left(\Pr_{\alg,\inst}[\event], \Pr_{\alg,\instnew}[\event]\right)\\
\ge& d(1-\delta,1/(n+1))\\
\ge& (1-\delta)\ln n\text{.}
\end{split}\end{equation*}
Therefore we conclude that
$$\Ex_{\alg,\inst}[\tau]\ge2(1-\delta)\Delta^{-2}\ln n\ge c\Delta^{-2}\ln n\text{.}$$

\end{proof}

Given Lemma~\ref{lem:symlb}, Lemma \ref{lem:genlb} may appear 
to be quite intuitive, as the symmetric instance $\instsym$ seems to be the worst case.
However, a rigorous proof of Lemma \ref{lem:genlb} 
is still quite nontrivial and is in fact the most technical part of the lower bound proof.
The proof consists of several steps which transform a general instance $\inst$ of \bestarm{} to a symmetric instance $\instsym$.

Suppose that some algorithm $\alg$ violates Lemma \ref{lem:genlb} on a \bestarm{} instance $\inst$. We divide the interval $[\mu - \Delta,\mu)$ into $n^{0.9}$ short segments, then at least one segment contains $n^{0.1}$ arms. We construct a smaller and denser instance $\instden$ consisting of the optimal arm and $n^{0.1}$ arms from the same segment. By Instance Embedding, there exists algorithm $\algnew$ that solves $\instden$ while taking few samples on the optimal arm. Note that the reduction crucially relies on the fact that since our lower bound is logarithmic in $n$, the bound merely shrinks by a constant factor after the number of arms decreases to $n^{0.1}$. 

Finally, we transform $\instden$ into a symmetric \bestarm{} instance $\instsym$ consisting of the optimal arm in $\instden$ along with $n^{0.1}$ copies of one of the sub-optimal arms. We also define an algorithm $\algsym$ that solves $\instsym$ with few samples drawn from the optimal arm, thus contradicting Lemma \ref{lem:symlb}. The full proof of Lemma \ref{lem:genlb} is postponed to Appendix C.


\section{UPPER BOUND}
\subsection{Building Blocks}
We start by introducing three subroutines that are useful for building our algorithm for \bestkarms{}.

\textbf{PAC algorithm for \bestkarms{}.} $\PAC$ is a PAC algorithm for \bestkarms{} adapted from the \PACSP{} algorithm in \cite{chen2016pure}. $\PAC$ is guaranteed to partition the given arm set into two sets $\sbig$ and $\ssma$, such that $\sbig$ approximates the best $k$ arms with high probability.

\begin{Lemma}\label{fact:PAC}
$\PAC(S, k, \eps, \delta)$ takes $$O\left(|S|\eps^{-2}\left[\ln\delta^{-1}+\ln\min(k,|S|-k)\right]\right)$$ samples and returns a partition $(\sbig,\ssma)$ of $S$ with $|\sbig|=k$ and $|\ssma| = |S| - k$. Let $\Mean{k}$ and $\Mean{k+1}$ denote the the $k$-th and the $(k+1)$-th largest means in $S$. With probability $1 - \delta$, it holds that
\begin{equation}\label{eq:PAC1}
\mu_{\arm}\ge\Mean{k} - \eps\text{ for all }\arm\in \sbig\text{,}
\end{equation}
\begin{equation}\label{eq:PAC2}
\mu_{\arm}\le\Mean{k+1}+\eps\text{ for all }\arm\in \ssma\text{.}
\end{equation}
\end{Lemma}

Lemma \ref{fact:PAC} is proved in Appendix D. We say that a specific call to $\PAC$ returns correctly if both \eqref{eq:PAC1} and \eqref{eq:PAC2} hold.




\textbf{PAC algorithms for \bestarm{}.} $\MEbig$ and $\MEsma$ approximate the largest and the smallest mean among several arms respectively. Both algorithms can be easily implemented by calling $\PAC$ with $k = 1$, and then sampling the best arm identified by $\PAC$.

\begin{Lemma}\label{fact:ME}
Both $\MEbig(S, \eps, \delta)$ and $\MEsma(S, \eps, \delta)$ take $O(|S|\eps^{-2}\ln\delta^{-1})$ samples and output a real number. Each of the following inequalities holds with probability $1 - \delta$:
\begin{equation}\label{eq:ME1}
\left|\MEbig(S, \eps, \delta) - \max_{\arm\in S}\mu_{\arm}\right|\le\eps
\end{equation}
\begin{equation}\label{eq:ME2}
\left|\MEsma(S, \eps, \delta) - \min_{\arm\in S}\mu_{\arm}\right|\le\eps
\end{equation}
\end{Lemma}

Lemma \ref{fact:ME} is proved in Appendix D. We say that a specific call to $\MEbig$ (or $\MEsma$) returns correctly if inequality \eqref{eq:ME1} (or \eqref{eq:ME2}) holds.

\textbf{Elimination procedures.} Finally, $\ELbig$ and $\ELsma$ are two elimination procedures. Roughly speaking, $\ELbig$ guarantees that after the elimination, the fraction of arms with means above the larger threshold $\thetabig$ is bounded by a constant. Meanwhile, a fixed arm with mean below the smaller threshold $\thetasma$ are unlikely to be eliminated. Analogously, $\ELsma$ removes arms with means below $\thetasma$, and preserves arms above $\thetabig$. The properties of $\ELbig$ and $\ELsma$ are formally stated below.

\begin{Lemma}\label{fact:EL}
Both $\ELbig(S, \thetasma, \thetabig, \delta)$ and $\ELsma(S, \thetasma, \thetabig, \delta)$ take $O(|S|\eps^{-2}\ln\delta^{-1})$ samples and return a set $T\subseteq S$. For $\ELbig$ and a fixed arm $\arm^*\in S$ with $\mu_{\arm^*}\le\thetasma$, it holds with probability $1 - \delta$ that $\arm^*\in T$ and
\begin{equation}\label{eq:EL1}
\left|\{\arm\in T:\mu_{\arm}\ge\thetabig\}\right|\le|T|/10\text{.}
\end{equation}
Similarly, for $\ELsma$ and fixed $\arm^*\in S$ with $\mu_{\arm^*}\ge\thetabig$, it holds with probability $1 - \delta$ that $\arm^*\in T$ and
\begin{equation}\label{eq:EL2}
\left|\{\arm\in T:\mu_{\arm}\le\thetasma\}\right|\le|T|/10\text{.}
\end{equation}
\end{Lemma}

Lemma \ref{fact:EL} is proved in Appendix D. We say that a call to $\ELbig$ (or $\ELsma$) returns correctly if inequality \eqref{eq:EL1} (or \eqref{eq:EL2}) holds.


\subsection{Algorithm}
Our algorithm for \bestkarms{}, \algo{}, is formally described below. \algo{} takes a parameter $k$, an instance $\inst$ of \bestkarms{} and a confidence level $\delta$ as input, and returns the best $k$ arms in $\inst$.

\begin{algorithm}\label{algo2}
\caption{\algo{}}
\KwIn{Parameter $k$, instance $\inst$, and confidence $\delta$.}
\KwOut{The best $k$ arms in $\inst$.}
$S_1\leftarrow \inst$; $T_1\leftarrow \emptyset$\;
\For{\upshape $r=1$ to $\infty$} {
    $\kbig_r\leftarrow k - |T_r|$; $\ksma_r\leftarrow |S_r| - \kbig_r$\;
    \lIf{\upshape $\kbig_r=0$}
        {\textbf{return} $T_r$}\label{line:ret1}
    \lIf{\upshape $\ksma_r=0$}
        {\textbf{return} $T_r\cup S_r$}\label{line:ret2}
    $\delta_r\leftarrow\delta/(20r^2)$\;
    $(\sbig_r, \ssma_r) \leftarrow \PAC(S_r, \kbig_r, \eps_r/8, \delta_r)$\;
    $\thetabig_r \leftarrow \MEbig(\ssma_r, \eps_r/8, \delta_r)$\;
    $\thetasma_r \leftarrow \MEsma(\sbig_r, \eps_r/8, \delta_r)$\;
    $\deltap_r\leftarrow \delta / \min(\kbig_r,\ksma_r)$\;
    $S_{r+1}\leftarrow \ELbig(\sbig_r, \thetabig_r + \eps_r / 8, \thetabig_r + \eps_r / 4, \deltap_r) \cup \ELsma(\ssma_r, \thetasma_r - \eps_r / 4, \thetasma_r - \eps_r / 8, \deltap_r)$\;\label{line:EL}
    $T_{r+1}\leftarrow T_r\cup\left(\sbig_r\setminus S_{r+1}\right)$\;
}
\end{algorithm}


Throughout the algorithm, \algo{} maintains two sets of arms $S_r$ and $T_r$ for each round $r$. $S_r$ contains the arms that are still under consideration at the beginning of round $r$, while $T_r$ denotes the set of arms that have been included in the answer. We say that an arm is removed (or eliminated) at round $r$, if it is in $S_r\setminus S_{r+1}$. Note that we may remove an arm either because its mean is so small that it cannot be among the best $k$ arms, or its mean is large enough so that we decide to include it in the answer. This justifies the name of our algorithm, \algo{}.

In each round $r$, \algo{} performs the following four steps.

\textbf{Step 1: Initialization.} \algo{} first calculates $\kbig_r$ and $\ksma_r$, which indicate that it needs to identify the $\kbig_r$ largest arms (or equivalently, the $\ksma_r$ smallest arms) in $S_r$. In the base case that either $\kbig_r = 0$ or $\ksma = 0$, it directly returns the answer.

\textbf{Step 2: Find a PAC solution.} Then \algo{} calls $\PAC$ to partition $S_r$ into $\sbig_r$ and $\ssma_r$ with size $\kbig_r$ and $\ksma_r$ respectively, such that $\sbig_r$ denotes an approximation of the best $\kbig_r$ arms in $S_r$.

\textbf{Step 3: Estimate Thresholds.} After that, \algo{} calls $\MEbig$ and $\MEsma$ to compute two thresholds $\thetabig_r$ and $\thetasma_r$. $\thetabig_r$ is an estimation of the largest mean in $\ssma_r$, which is approximately the mean of the $(\kbig_r+1)$-th largest arm in $S_r$. Analogously, $\thetasma_r$ approximates the $\kbig_r$-th largest mean in $S_r$.

It might seem weird at first glance that $\thetabig_r$ and $\thetasma_r$ approximates the $(\kbig_r+1)$-th mean and the $\kbig_r$-th mean respectively, implying that $\thetabig_r$ is expected to be \textit{smaller} than $\thetasma_r$. In fact, the superscript ``\textsf{large}'' in $\thetabig_r$ indicates that it is the threshold used for eliminating arms in $\sbig_r$.


\textbf{Step 4: Elimination.} Finally, \algo{} calls $\ELbig$ and $\ELsma$ to eliminate the arms in $\sbig_r$ that are significantly larger than $\thetabig_r$, and the arms in $\ssma_r$ that are much smaller than $\thetasma_r$. The arms removed from $\sbig_r$ are included into the answer.

\eat{
\mingda{The four steps above seem to be just explaining the pseudo-codes in natural language... I'm not very confident about this part... Maybe I should also try a top-bottom way to intuitively explain our algorithm?} 

\lijie{I think this is already a very good description.}
}

\textbf{Caveats.} Note that our algorithm uses a different confidence level, $\deltap_r$, in Step 4. Intuitively, at most $\min(\kbig_r,\ksma_r)$ arms among the best $\kbig_r$ arms in $S_r$ are misclassified as ``small arms'' by $\PAC$. Therefore during the elimination process, it is crucial that such misclassified arms are not mistakenly eliminated. As a result, we need a union bound on these arms, which contributes to the $\min(\kbig_r, \ksma_r)$ factor in our confidence level.

\eat{It is worth pointing out that a naive union bound on all the $\kbig_r + \ksma_r$ arms would lead to an inferior sample complexity. For example, consider the case that $\kbig_r = 1$ and $\ksma_r = n$. If we naively set $\deltap_r$ to be $\delta_r/n$, we would pay an extra cost of $n\eps_r^{-2}\ln n$ in the $\EL$ procedure. On the other hand, our algorithm would simply set $\deltap_r = \delta_r$ in this case, thus avoiding the waste of samples.}

\subsection{Observations}
We start our analysis of \algo{} with a few simple yet useful observations.

\textbf{Good events.} We define $\goodevent_r$ as the event that in round $r$, all the five calls to $\PAC$, $\ME$, and $\EL$ return correctly. These events are crucial to our following analysis, as they guarantee that the partition $(\sbig_r,\ssma_r)$ and thresholds $\thetabig_r$ and $\thetasma_r$ are sufficiently accurate, and additionally, $\EL$ eliminates a sufficiently large fraction of arms. The following observation, due to a simple union bound, lower bounds the probability of each good event.

\begin{Observation}\label{obs:event}
$\Pr[\goodevent_r]\ge 1 - 5\delta_r$.
\end{Observation}

\textbf{Valid executions.} We say that an execution of \algo{} is \textit{valid} at round $r$, if and only if the following two conditions are satisfied:
\begin{itemize}
\item For each $1\le i < r$, event $\goodevent_{i}$ happens. (i.e., all calls to subroutines return correctly in previous rounds.) 
\item The union of $T_r$ and the best $\kbig_r$ arms in $S_r$ is the correct answer of the \bestkarms{} instance. In other words, no arms have been incorrectly eliminated in previous rounds.
\end{itemize}
Moreover, an execution is valid if it is valid at every round before it terminates. We define $\validevent$ to be the event that an execution of \algo{} is valid. 

\textbf{Thresholds. } In the following, we bound the thresholds $\thetabig_r$ and $\thetasma_r$ returned by subroutine $\ME$ conditioning on $\goodevent_r$. Let $\mubig_r$ and $\musma_r$ denote the means of the $\kbig_r$-th and the $(\kbig_r+1)$-th largest arms in $S_r$. We show that $\thetabig_r$ and $\thetasma_r$ are $O(\eps_r)$-approximations of $\musma_r$ and $\mubig_r$ conditioning on the good event $\goodevent_r$. The proof of the following observation is postponed to Appendix D.

\begin{Observation}\label{obs:theta}
Conditioning on $\goodevent_r$,
$$\thetabig_r\in\left[\musma_r - \eps_r/8, \musma_r + \eps_r/4\right]\text{,}$$
$$\thetasma_r\in\left[\mubig_r - \eps_r/4, \mubig_r + \eps_r/8\right]\text{.}$$
\end{Observation}

\textbf{Number of remaining arms.} Finally, we show that conditioning on the validity of an execution, the number of remaining arms at the beginning of each round can be upper bounded in terms of $|\gbig_{\ge r}|$ and $|\gsma_{\ge r}|$. The following observation, proved in Appendix D, is crucial to analyzing the sample complexity of our algorithm.

\begin{Observation}\label{obs:kbound}
Conditioning on $\validevent$, it holds that $\kbig_r\le2|\gbig_{\ge r}|$ and $\ksma_r\le2|\gsma_{\ge r}|$.
\end{Observation}


\subsection{Correctness}
Recall that $\validevent$ denotes the event that the execution of \algo{} is valid. The following lemma, proved in Appendix D, shows that event $\validevent$ happens with high probability.

\begin{Lemma}\label{lem:valid}
$\Pr\left[\validevent\right]\ge 1 - \delta$.
\end{Lemma}

We show that \algo{} always returns the correct answer conditioning on $\validevent$, thus proving that \algo{} is \CORRECT{}.

\begin{Lemma}\label{lem:correct}
\algo{} returns the correct answer with probability at least $1-\delta$.
\end{Lemma}

\begin{proof}[Proof of Lemma \ref{lem:correct}]

It suffices to show that conditioning on $\validevent$, the algorithm always returns the correct answer. In fact, if \algo{} terminates at round $r$, it either returns $T_r$ at Line \ref{line:ret1} or returns $T_r\cup S_r$ at Line \ref{line:ret2}. According to the second property guaranteed by the validity at round $r$, the answer returned by \algo{} must be correct.

It remains to show that \algo{} does not run forever. Recall that $\Gap{k} = \Mean{k} - \Mean{k+1}$ is the gap between the $k$-th and the $(k+1)$-th largest means in the original instance $\inst$. We choose a sufficiently large $r^*$ that satisfies $\eps_{r^*} < \Gap{k}$. By definition, we have $\gbig_{\ge r^*} = \gsma_{\ge r^*} = \emptyset$. Then Observation \ref{obs:kbound} implies that $\kbig_{r^*} = \ksma_{r^*} = 0$, if the algorithm does not terminate before round $r^*$. Therefore the algorithm either terminates at or before round $r^*$. This completes the proof.
\end{proof}

\subsection{Sample Complexity}
We prove the following Lemma \ref{lem:sample}, which bounds the sample complexity of \algo{} conditioning on $\validevent$. Then Theorem \ref{theo:ub} directly follows from Lemma \ref{lem:correct} and Lemma \ref{lem:sample}.
\eat{The bound is relatively weaker than that required by Theorem \ref{theo:ub}, since Lemma \ref{lem:sample} does not rule out the possibility that the algorithm takes many (or even infinitely many) samples with a small probability. To address this subtlety, we apply the parallel simulation trick developed in \cite{chen2015optimal}, and then transform \algo{} into a \CORRECT{} algorithm with the desired expected sample complexity, thus proving Theorem \ref{theo:ub}.}
The proof of Theorem \ref{theo:ubrefine} is postponed to the appendix.

\begin{Lemma}\label{lem:sample}
Conditioning on event $\validevent$, \algo{} takes $O(H\ln\delta^{-1} + \htilbig + \htilsma + \htil)$ samples.
\end{Lemma}

\begin{proof}[Proof of Lemma \ref{lem:sample}]
We consider the $r$-th round of the algorithm. Recall that $\kbig_r + \ksma_r = |S_r|$. According to Lemmas \ref{fact:PAC} through \ref{fact:EL}, $\PAC$ takes
\begin{equation}\label{eq:sample}
O\left(|S_r|\eps_r^{-2}\left[\ln\delta_r^{-1} + \ln\min\left(\kbig_r,\ksma_r\right)\right]\right)\end{equation}
samples. $\MEbig$ and $\MEsma$ \eat{\lijie{actually it is $\MEbig$ and $\MEsma$?}} take
\begin{equation*}\begin{split}
O\left((\kbig_r + \ksma_r)\eps_r^{-2}\ln\delta_r^{-1}\right) = O\left(|S_r|\eps_r^{-2}\ln\delta_r^{-1}\right)
\end{split}\end{equation*}
samples in total, while $\ELbig$ and $\ELsma$ \eat{\lijie{same}} take
\begin{equation*}\begin{split}
&O\left(\kbig_r\eps_r^{-2}\ln{\deltap_r}^{-1}\right) + O\left(\ksma_r\eps_r^{-2}\ln{\deltap_r}^{-1}\right)\\
=&O\left(|S_r|\eps_r^{-2}\left[\ln\delta_r^{-1} + \ln\min\left(\kbig_r,\ksma_r\right)\right]\right)
\end{split}\end{equation*}
samples conditioning on $\validevent$. Clearly the sample complexity in round $r$ is dominated by \eqref{eq:sample}.

\textbf{Simplify and split the sum:} By Observation \ref{obs:kbound}, conditioning on event $\validevent$, $\kbig_r$ and $\ksma_r$ are bounded by $2\left|\gbig_{\ge r}\right|$ and $2\left|\gsma_{\ge r}\right|$ respectively. Thus it suffices to bound the sum of $\Hone_r + \Htwobig_r + \Htwosma_r$, where
$$\Hone_r = \left(|\gbig_{\ge r}|+|\gsma_{\ge r}|\right)\eps_r^{-2}(\ln\delta^{-1} + \ln r)\text{,}$$
$$\Htwobig_r = \eps_r^{-2}|\gbig_{\ge r}|\ln|\gsma_{\ge r}|\text{,}$$
$$\Htwosma_r = \eps_r^{-2}|\gsma_{\ge r}|\ln|\gbig_{\ge r}|\text{.}$$

In fact, since
$$\ln\delta_r^{-1} = \ln\delta^{-1} + \ln(20r^2) = O\left(\ln\delta^{-1} + \ln r\right)\text{,}$$ the $|S_r|\eps_r^{-2}\ln\delta_r^{-1}$ term in \eqref{eq:sample} is bounded by $\Hone_r$. Moreover, the $|S_r|\eps_r^{-2}\ln\min(\kbig_r,\ksma_r)$ term is smaller than or equal to
$$\eps_r^{-2}\left(\kbig_r\ln\ksma_r+\ksma_r\ln\kbig_r\right)\text{,}$$
and is thus upper bounded by $\Htwobig_r + \Htwosma_r$.

In Appendix D, we show with a straightforward calculation that
$$\sum_{r=1}^{\infty}\Hone_r = O\left(H\ln\delta^{-1} + \htil\right)\text{,}$$
$$\sum_{r=1}^{\infty}\Htwobig_r = O\left(\htilbig\right)\text{, and}$$
$$\sum_{r=1}^{\infty}\Htwosma_r = O\left(\htilsma\right)\text{.}$$
Then the lemma directly follows.
\end{proof}

\eat{
\textbf{Upper bound the $\Hone$ term:} It follows from a directly calculation that
\begin{equation*}\begin{split}
\sum_{r=1}^{\infty}\Hone_r
=&\sum_{r=1}^{\infty}\sum_{i = r}^{\infty}\left(|\gbig_{i}|+|\gsma_{i}|\right)\eps_r^{-2}(\ln\delta^{-1} + \ln r)\\
=&\sum_{i=1}^{\infty}\left(|\gbig_{i}|+|\gsma_{i}|\right)\sum_{r=1}^{i}\eps_r^{-2}(\ln\delta^{-1} + \ln r)\\
=&O\left(\sum_{i=1}^{\infty}\left(|\gbig_{i}|+|\gsma_{i}|\right)\eps_i^{-2}(\ln\delta^{-1} + \ln i)\right)\\
=&O\left(\sum_{i=1}^{n}\Gap{i}^{-2}\left(\ln\delta^{-1} + \ln\ln\Gap{i}^{-1}\right)\right)\text{.}
\end{split}\end{equation*}
Here the second step interchanges the order of summation. The third step holds since the inner summation is always dominated by the last term. Finally, the last step is due to the fact that $\Delta_{\arm} = \Theta(\eps_i)$ for every arm $\arm\in\gbig_i\cup\gsma_i$.
Therefore we have
$$\sum_{r=1}^{\infty}\Hone_r = O(H\ln\delta^{-1} + \htil)\text{.}$$

\textbf{Upper bound $\Htwobig$ and $\Htwosma$:}
By definition of $\Htwobig_r$, we have
\begin{equation*}\begin{split}
\sum_{r=1}^{\infty}\Htwobig_r
=&\sum_{r=1}^{\infty}\eps_r^{-2}|\gbig_{\ge r}|\ln|\gsma_{\ge r}|\\
=&\sum_{r=1}^{\infty}\sum_{i = r}^{\infty}\eps_r^{-2}|\gbig_{i}|\ln|\gsma_{\ge r}|\\
=&\sum_{i=1}^{\infty}|\gbig_{i}|\sum_{r=1}^{i}\eps_r^{-2}\ln|\gsma_{\ge r}|\text{.}
\end{split}\end{equation*}

Therefore we conclude that $$\sum_{r=1}^{\infty}\Htwobig_r = O(\htilbig)\text{.}$$ The bound on the sum of $\Htwosma_r$ follows from an analogous calculation.}

Finally, we prove our main result on the upper bound side.
\begin{proof}[Proof of Theorem \ref{theo:ub}]
Let $$T = H\ln\delta^{-1} + \htil + \htilbig + \htilsma\text{.}$$
Lemma \ref{lem:correct} and Lemma \ref{lem:sample} together imply that conditioning on an event that happens with probability $1-\delta$, \algo{} returns the correct answer and takes $O(T)$ samples. 
Using the parallel simulation trick in \cite[Theorem H.5]{chen2015optimal},
we can obtain an algorithm which uses $O(T)$ samples in expectation (unconditionally),
thus proving Theorem \ref{theo:ub}.
\end{proof}


\eat{\subsubsection*{Acknowledgements}

Use unnumbered third level headings for the acknowledgements.  All
acknowledgements go at the end of the paper.  Be sure to omit any
identifying information in the initial double-blind submission!}

\subsubsection*{References}
\bibliographystyle{alpha}
{\def\section*#1{}
\renewcommand{\refname}{}
\bibliography{team}}

\newpage

\appendix

\section*{Organization of the Appendix}

In the appendix, we present the missing proofs in this paper. In Appendix A, we first discuss a specific instance mentioned in Section 1, showing that our upper bound strictly improves previous algorithms. In Appendix B, we prove Fact \ref{fact:d} in Section 2. In Appendix C, we prove the Instance Embedding lemma (Lemma \ref{lem:InstEmb}) and the relatively technical Lemma \ref{lem:genlb}, which relates a general instance of \bestarm{} to a symmetric instance. In Appendix D, we discuss the implementation of the building blocks of our algorithm, prove a few useful and observations, and finally complete the missing proofs of other lemmas and theorems.

\section{Specific \bestkarms{} Instance}
We show that our upper bound results (Theorem \ref{theo:ub} and Theorem \ref{theo:ubrefine}) strictly improve the state-of-the-art algorithm for \bestkarms{} obtained in \cite{chen2016pure} by calculating the sample complexity of both algorithms on a specific \bestkarms{} instance.

We consider a family of instances parametrized by integer $n$ and $\eps \in (0, 1/4)$. Each instance consists of $n$ arms with mean $0$, $n$ arms with mean $1/2$, along with two arms with means $1/4+\eps$ and $1/4-\eps$ respectively. We are required to identify the top $n+1$ arms. By definition, the gap of every arm with mean $0$ or $1/2$ is $1/4 + \eps$, while the gaps of the remaining two arms are $2\eps$. As $\eps$ tends to zero, the arms with gap $1/4 + \eps$ become relatively simple: an algorithm can decide whether to include them in the answer or not with few samples. The hardness of the instance is then concentrated on the two arms with close means.

For simplicity, we assume that the confidence level, $\delta$, is set to a constant. Then the $O(H\ln\delta^{-1})$ term in the upper bounds are dominated by the $O(\htil)$ term. By a direct calculation, we have
$$\htil = \Theta(n + \eps^{-2}\ln\ln\eps^{-1})\text{.}$$

Let $m$ be the integer that satisfies $2\eps\in(\eps_{m+1},\eps_m]$. Then we have $$|\gbig_1| = |\gsma_1| = n\text{, and}$$
$$|\gbig_{m}| = |\gsma_m| = 1\text{.}$$
It follows from the definition of $\htilbig$ and $\htilsma$ that
$$\htilbig = \htilsma = O(n\ln n + \eps^{-2})\text{.}$$

By Theorem \ref{theo:ub}, for constant $\delta$, our algorithm takes 
$$O(\htil + \htilbig + \htilsma) = O(n\ln n + \eps^{-2}\ln\ln\eps^{-1})$$
samples on this instance.

On the other hand, the upper bound achieved by \PACSP{} algorithm is
$$O(\htil + H\ln n) = O(n\ln n + \eps^{-2}\ln\ln\eps^{-1} + \eps^{-2}\ln n)\text{.}$$

Note that if $\eps = 1/n$, our algorithm takes $O(n^2\ln\ln n)$ samples, while \PACSP{} takes $O(n^2\ln n)$ samples. This indicates that there is a logarithmic gap between the state-of-the-art upper bound and the instance-wise lower bound, while we narrow down the gap to a doubly-logarithmic factor.

\section{Missing Proof in Section 2}
\textbf{Fact~\ref{fact:d}}
(restated)
{\em	
For $0\le y\le y_0\le x_0\le x\le 1$, $d(x,y)\ge d(x_0,y_0)$.
}
\begin{proof}[Proof of Fact \ref{fact:d}]
Taking the partial derivative yields
$$\frac{\partial d(x,y)}{\partial x} = \ln\frac{x(1-y)}{y(1-x)}\text{,}$$
$$\frac{\partial d(x,y)}{\partial y} = \frac{y-x}{y(1-y)}\text{.}$$
Therefore when $x\ge y$, $d(x,y)$ is increasing in $x$ and decreasing in $y$, which proves the fact.
\end{proof}

\section{Missing Proofs in Section 3}
\subsection{Proof of Lemma \ref{lem:InstEmb}}
\textbf{Lemma~\ref{lem:InstEmb}}
(restated)
{\em	
Let $\inst$ be a \bestkarms{} instance. Let $\arm$ be an arm among the top $k$ arms, and $\instemb$ be a \bestarm{} instance consisting of $\arm$ and a subset of arms in $\inst$ outside the top $k$ arms. If some algorithm $\alg$ solves $\inst$ with probability $1-\delta$ while taking less than $N$ samples on $\arm$ in expectation, there exists another algorithm $\algemb$ that solves $\instemb$ with probability $1-\delta$ while taking less than $N$ samples on $\arm$ in expectation.
}
\begin{proof}[Proof of Lemma \ref{lem:InstEmb}]
We construct the following algorithm $\algemb$ for $\instemb$. Given the instance $\instemb$, $\algemb$ first augments the instance into $\inst$ by adding a fictitious arm for each arm in $\inst\setminus\instemb$. Then $\algemb$ simulates $\alg$ on the \bestkarms{} instance $\inst$. When $\alg$ requires a sample from an arm in $\instemb$, $\algemb$ draws a sample and sends it to $\alg$. If $\alg$ requires a sample from an arm outside $\instemb$, $\algemb$ generates a fictitious sample on its own and then sends it to $\alg$. When $\alg$ terminates and returns a subset $S$ of $k$ arms, $\algemb$ terminates and returns an arbitrary arm in $S\cap\instemb$.

Note that when $\algemb$ runs on instance $\instemb$, the algorithm $\alg$ simulated by $\algemb$ effectively runs on the instance $\inst$. It follows that with probability $1-\delta$, $\alg$ returns the correct answer of the \bestkarms{} instance $\inst$, and thus $\arm$ is the only arm in both $\instemb$ and the set $S$ returned by $\alg$. Therefore, $\algemb$ correctly solves the \bestarm{} instance $\instemb$ with probability at least $1-\delta$. Moreover, the expected number of samples drawn from arm $\arm$ is less than $N$ by our assumptions.
\end{proof}

\subsection{Proof of Lemma \ref{lem:genlb}}
\textbf{Lemma~\ref{lem:genlb}}
(restated)
{\em	
Let $\inst$ be an instance of \bestarm{} consisting of one arm with mean $\mu$ and $n$ arms with means on $[\mu - \Delta,\mu)$. There exist universal constants $\delta$ and $c$ (independent of $n$ and $\Delta$) such that for all algorithm $\alg$ that correctly solves $\inst$ with probability $1-\delta$, the expected number of samples drawn from the optimal arm is at least $c\Delta^{-2}\ln n$.
}
\begin{proof}[Proof of Lemma \ref{lem:genlb}]
Let $\delta_0$ and $c_0$ be the constants in Lemma \ref{lem:symlb}. We claim that Lemma \ref{lem:genlb} holds for constants $\delta = \delta_0/3$ and $c = c_0\delta_0/30$.

Suppose for a contradiction that when algorithm $\alg$ runs on \bestarm{} instance $\inst$, it outputs the correct answer with probability $1-\delta$ and the optimal arm $\arm_0$ is sampled less than $c\Delta^{-2}\ln n$ times in expectation.

\textbf{Overview.} Our proof follows the following five steps.

Step 1. We apply Instance Embedding to obtain a smaller yet denser (in the sense that all suboptimal arms have almost identical means) instance $\instden$, together with a new algorithm $\algnew$ that solves $\instden$ by taking few samples on the optimal arm with high probability.

Step 2. We obtain a symmetric instance $\instsym$ from $\instden$ by making the suboptimal arms identical to each other. We also define an algorithm $\algsym$ for instance $\instsym$.

Step 3. To analyze algorithm $\algsym$ on instance $\instsym$, we define the notion of ``mixed arms'', which return a fixed number of samples from one distribution, and then switch to another distribution permanently. We transform $\instden$ into an intance $\instmix$ with mixed arms.

Step 4. We show by Change of Distribution that when $\algnew$ runs on $\instmix$, it also returns the correct answer with few samples on the optimal arm.

Step 5. We show that the execution of $\algsym$ on $\instsym$ is, in a sense, equivalent to the execution of $\algnew$ on $\instmix$. This finally leads to a contradiction to Lemma \ref{lem:symlb}.

The reductions involved in the proof is illustrated in Figure 1.

\begin{figure*}\label{fig:reduction}
	\centering
	\begin{tikzpicture}[->,>=stealth',shorten >=1pt,auto,
	semithick,scale = 2.0]
	\tikzstyle{every state}=[text=black, rectangle]
	\tikzstyle{solve}=[thick]
	\node [state] (1) at (0, 0) {$\alg$ on $\inst$};
	\node [state] (2) at (2.3, 0) {$\algnew$ on $\instden$};
	\node [state] (3) at (5, 0) {\begin{tabular}{c}$\algnew$ on $\instmix$\\($\exmix$)\end{tabular}};
	\node [state] (4) at (7.5, 0) {\begin{tabular}{c}$\algsym$ on $\instsym$\\($\exsym$)\end{tabular}};

	\path	(1)	edge[solve, above]	node	{Step 1}						(2)
			(1)	edge[solve, below]	node	{{\begin{tabular}{c}Instance\\Embedding\end{tabular}}}		(2)
	      	(2)	edge[solve, above]	node	{Step 4}						(3)
	      	(2)	edge[solve, below]	node	{\begin{tabular}{c}Change of\\Distribution\end{tabular}}	(3)
	      	(3)	edge[solve, above]	node	{Step 5}						(4)
	      	(3)	edge[solve, below]	node	{Equivalence}					(4);
	\end{tikzpicture}
	
	\caption{Each rectangle denotes the execution of an algorithm on an instance. The arrows specify the step in which each reduction is performed and the major technique involved in the reduction.}
\end{figure*}
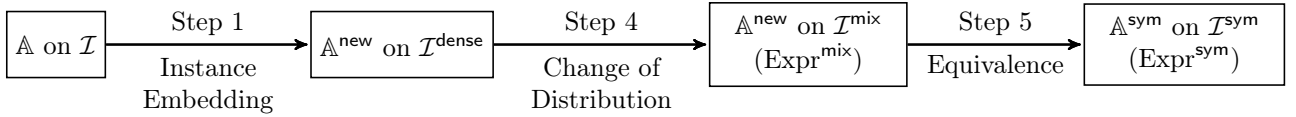

\textbf{Step 1: Construct $\instden$ and $\algnew$.} We first construct a new \bestarm{} instance $\instden$ in which the sub-optimal arms have almost identical means. Let $\mu_0$ denote the mean of the optimal arm $\arm_0$. We divide the interval $[\mu_0 - \Delta, \mu_0]$ into $n^{0.9}$ segments, each with length $\Delta/n^{0.9}$. Set $m = n^{0.1}$. By the pigeonhole principle, we can assume that $\arm_1,\arm_2,\ldots,\arm_m$ are $m$ arms with means in the same interval. Let $\mu_i$ denote the mean of arm $\arm_i$. By construction, $\mu_0 - \mu_i\le\Delta$ for all $1\le i\le m$ and $\left|\mu_i - \mu_j\right|\le\Delta/n^{0.9}$ for all $1\le i,j\le m$.

We simply let $\instden = \{\arm_0,\arm_1,\arm_2,\ldots,\arm_{m}\}$. By Instance Embedding (Lemma \ref{lem:InstEmb}), there exists an algorithm $\algnew$ that solves $\instden$ with probability $1-\delta$ while taking less than $c\Delta^{-2}\ln n$ samples on $\arm_0$ in expectation. We will focus on instance $\instden$ in the rest of our proof.

Recall that $\Pr_{\alg,\inst}$ and $\Ex_{\alg,\inst}$ denote the probability and expectation when algorithm $\alg$ runs on instance $\inst$ respectively.
Let $\tau_i$ denote the number of samples taken on $A_i$. Then we have
$$\Ex_{\algnew, \instden}[\tau_0] \le c\Delta^{-2}\ln n\text{.}$$
Let $N = c\delta^{-1}\Delta^{-2}\ln n$.
By Markov's inequality,
$$\Pr_{\algnew, \instden}[\tau_0 \ge N]\le\frac{c\Delta^{-2}\ln n}{N}=\delta\text{.}$$
Let $\event$ denote the event that the algorithm returns the correct answer while taking at most $N$ samples on arm $\arm_0$. The union bound implies that
$$\Pr_{\algnew, \instden}[\event] \ge 1-2\delta\text{.}$$


\textbf{Step 2: Construct $\instsym$ and $\algsym$.} Let $\instsym$ be the \bestarm{} instance consisting of arm $\arm_0$ and $m = n^{0.1}$ copies of arm $\arm_1$. We define algorithm $\algsym$ as follows. Given instance $\instsym$, $\algsym$ simulates algorithm $\algnew$ as if $\algnew$ is running on instance $\instden$.\eat{ on instance $\instden$}
When $\algnew$ requires a sample from an arm $\arm$ that has not been pulled $N$ times (recall that $N = c\delta^{-1}\Delta^{-2}\ln n$), $\algsym$ draws a sample from $\arm$ and sends it to $\algnew$. When the number of pulls on $\arm$ exceeds $N$ for the first time, $\algsym$ assigns a random number $\pi(\arm)$ from $\{1,2,\ldots,m\}$ to arm $\arm$, such that $\pi(\arm)$ is different from every number that has already been assigned to another arm. If this step cannot be performed because all numbers in $\{1,2,\ldots,m\}$ have been used up, $\algsym$ simply terminates without returning an answer.\footnote{As shown in the analysis in Step 5, we only care the behavior of $\algsym$ when the labels are not used up.} After that, upon each pull of $\arm$, $\algsym$ sends a sample drawn from $\normal(\mu_{\pi(A)},1)$ to $\algnew$. (Recall that $\mu_i$ denotes the mean of arm $\arm_i$ in $\instden$.) Finally, $\algsym$ outputs what $\algnew$ outputs.

\eat{
\lijie{it seems it did not consider the case where the labels in $\{1, \dotsc,m \}$ are used up? I know that case does not affect anything though, but still need to explain a word}
\mingda{Fixed.}
}

\textbf{Step 3: Construct mixed arms and $\instmix$.} In order to analyze the execution of $\algsym$ on instance $\instsym$, it is helpful to define $m$ ``mixed arms''. For $1\le i\le m$, the $i$-th mixed arm, denoted by $\marm_i$, returns a sample drawn from $\normal(\mu_1, 1)$ (i.e., the reward distribution of arm $\arm_1$) when it is pulled for the first $N$ times. After $N$ pulls, $\marm_i$ returns samples from $\normal(\mu_i, 1)$ as $\arm_i$ does. For ease of notation, we also let $\marm_0$ denote $\arm_0$. Let $\instmix$ denote the \bestarm{} instance $\{\marm_0,\marm_1,\marm_2,\ldots,\marm_m\}$. 

\textbf{Step 4: Run $\algnew$ on $\instmix$.} Now suppose we run $\algnew$ on instance $\instmix$. In fact, we may view each arm (either $\arm_i$ or $\marm_i$) as two separate ``semi-arms''. When $\algnew$ samples arm $\arm_i$ in the first $N$ times, it pulls the first semi-arm of $\arm_i$. After $\arm_i$ has been pulled $N$ times, $\algnew$ pulls the second semi-arm. From this perspective, $\instmix$ is simply obtained from $\instden$ by changing the first semi-arm of each arm $\arm_i$ ($1\le i\le m$) from $\normal(\mu_i,1)$ to $\normal(\mu_1,1)$. Since the first semi-arm is sampled at most $N$ times by $\algnew$, it follows from Change of Distribution (Lemma \ref{lem:CoD}) that
\begin{equation*}\begin{split}
&d\left(\Pr_{\algnew, \instden}[\event],\Pr_{\algnew, \instmix}[\event]\right)\\
\le & \sum_{i=1}^{m} N\cdot \KL\left(\normal(\mu_i,1),\normal(\mu_1,1)\right)\\
=& \frac{N}{2}\sum_{i=1}^{m}(\mu_i-\mu_1)^2\\
\le & \frac{c\delta^{-1}\Delta^{-2}\ln n}{2}\cdot n^{0.1}\cdot(\Delta/n^{0.9})^2\le \frac{c}{2\delta}n^{-1.7}\ln n\text{.}
\end{split}\end{equation*}
Here the second step follows from
$$\KL(\normal(\mu_1,1),\normal(\mu_2,1)) = (\mu_1-\mu_2)^2/2\text{.}$$
The third step is due to $N = c\delta^{-1}\Delta^{-2}\ln n$, $m = n^{0.1}$, and $|\mu_1 - \mu_i|\le \Delta/n^{0.9}$.

For sufficiently large $n$, we have
$$\frac{c}{2\delta}n^{-1.7}\ln n < d(1-2\delta, 1-3\delta)\text{.}$$
Recall that $\Pr_{\algnew,\instden}[\event]\ge 1-2\delta$. It follows from the monotonicity of $d(\cdot,\cdot)$ (Fact \ref{fact:d}) that $$\Pr_{\algnew, \instmix}[\event]\ge 1-3\delta\text{.}$$

\eat{
\lijie{you use a property of $d(\cdot,\cdot)$ here, better to remind the reader that...}
\mingda{Fixed.}
}

\textbf{Step 5: Analyze $\algsym$ and derive a contradiction to Lemma \ref{lem:symlb}.} For clarity, let $\exmix$ denote the experiment that $\algnew$ runs on $\instmix$, and $\exsym$ denote the experiment that $\algsym$ runs on $\instsym$. Step 4 implies that event $\event$ happens with probability at least $1 - 3\delta$ in experiment $\exmix$.

In the following, we derive the likelihood of an arbitrary execution of $\exmix$ in which event $\event$ happens, and prove that this execution has the same likelihood in experiment $\exsym$. As a result, $\algsym$ also returns the correct answer with probability at least $1-3\delta$. Moreover, according to our construction, $\algsym$ always takes at most $N$ samples on arm $\arm_0$. On the other hand, since $\mu_0 - \mu_1\le\Delta$, Lemma \ref{lem:symlb} implies that no algorithm can solve $\algsym$ correctly with probability $1 - \delta_0 = 1-3\delta$ while taking less than
$$c_0\Delta^{-2}\ln m = 30c\delta_0^{-1}\cdot\Delta^{-2}\cdot(0.1\ln n) = N$$
samples on $\arm_0$ in expectation. This leads to a contradiction and finishes the proof.

\textbf{Technicalities: equivalence between $\exmix$ and $\exsym$.} For ease of notation, we assume in the following that algorithm $\algnew$ is deterministic.\footnote{In fact this assumption is without loss of generality: the argument still holds conditioning on the randomness of $\algnew$.} Then the only randomness in experiment $\exmix$ stems from the random permutation of arms at the beginning, and the samples drawn from the arms.

We consider an arbitrary run of experiment $\exmix$ in which event $\event$ happens (i.e., $\algnew$ returns the optimal arm before taking more than $N$ samples from it).  For $0\le i\le m$, let $\sigma(i)$ denote the index of the $i$-th arm received by algorithm $\algnew$. (i.e., the $i$-th arm received by $\algnew$ is $\marm_{\sigma(i)}$.) By definition, $\sigma$ is a uniformly random permutation of $\{0,1,\ldots,m\}$. Let $\obs_i$ denote the sequence of samples that $\algnew$ observes from the $i$-th arm. Then the likelihood of this execution is given by
\begin{equation}\label{eq:likeli1}
\frac{1}{(m+1)!}\sum_{\sigma}\prod_{i=0}^{m}\pdf{\marm_{\sigma(i)}}{\obs_i}\text{.}
\end{equation}
Here we sum over all random permutations $\sigma$ on $\{0,1,2,\ldots,m\}$, and $\pdf{\marm_{\sigma(i)}}{\obs_i}$ denote the probability density of observing $\obs_i$ on arm $\marm_{\sigma(i)}$.

Now we compute the likelihood that in experiment $\exsym$, the algorithm $\algnew$ simulated by $\algsym$ observes the same sequence of samples. Let $\lambda$ denote the random permutation of arms given to $\algsym$. We define
$$\pos = \lambda^{-1}(0)\text{,}$$
$$\Long = \{i\in\{0,1,2,\ldots,m\}: |\obs_i| > N\}\text{, and}$$
$$\Short = \{0,1,\ldots,m\}\setminus\left(\Long\cup\{\pos\}\right)\text{.}$$
In other words, $\pos$ is the position of the optimal arm $\arm_0$ in $\instsym$. $\Long$ denote the positions of suboptimal arms that have been sampled more than $N$ times, while $\Short$ denote the remaining suboptimal arms. Note that since less than $N$ samples are taken on the optimal arm, $\pos$ is excluded from both sets.

Another source of randomness in $\exsym$ is the random numbers $\pi(\cdot)$ that $\algsym$ assigns to different arms. In this specific execution, function $\pi(\cdot)$ chosen by $\algsym$ is a random injection from $\Long$ to $\{1,2,\ldots,m\}$. By our construction of $\algsym$, for each $i\in\Long$, the algorithm $\algnew$ simulated by $\algsym$ first observes $N$ samples drawn from $\normal(\mu_1, 1)$ (i.e., the reward distribution of arm $\arm_1$) on the $i$-th arm. After that, $\algnew$ starts to observe samples drawn from $\normal(\mu_{\pi(i)},1)$. Recall that the mixed arm $\marm_{\pi(i)}$ also returns samples in this pattern. Therefore, the likelihood of observations on the $i$-th arm is exactly
\begin{equation}\label{eq:pdf}
\pdf{\marm_{\pi(i)}}{\obs_i}\text{.}
\end{equation}

In fact, we may express the likelihood for all arms as in \eqref{eq:pdf} by extending $\pi$ into a permutation on $\{0,1,2,\ldots,m\}$. First, we set $\pi(\pos) = 0$. Recall that the optimal arm is sampled less than $N$ times, all the samples observed from it are drawn from $\normal(\mu_0, 1)$, which is exactly the reward distribution of $\marm_0 = \marm_{\pi(\pos)}$. Therefore the likelihood of observations $\obs_{p^*}$ is given by
$$\pdf{\marm_{\pi(\pos)}}{\obs_{p^*}}\text{.}$$

Second, we let $\range = \{1,2,\ldots,m\}\setminus\pi(\Long)$ denote the available labels among $\{1,2,\ldots,m\}$. We define $\pi$ on $\Short$ by matching $\Short$ with $\range$ uniformly at random. Note that since all arms in $\Short$ are sampled at most $N$ times, $\algnew$ simulated by $\algsym$ always observes samples drawn from $\normal(\mu_1,1)$, which agrees with the first $N$ samples from every mixed arm $\marm_{i}$ ($i\ne 0$). Therefore, the likelihood of observations on the $i$-th arm where $i\in\Short$ is also given by
$$\pdf{\marm_{\pi(i)}}{\obs_i}\text{.}$$

According to our analysis above, the samples from the $i$-th arm observed by the simulated $\algnew$ in experiment $\exsym$ follows the same distribution as samples drawn from $\marm_{\pi(i)}$. Moreover, $\pi$ is a uniformly random permutation with the only condition that $\pi(\pos) = 0$, which is equivalent to $\pi^{-1}(0) = \pos = \lambda^{-1}(0)$.
Therefore, the likelihood is given by
\begin{equation}\label{eq:likeli2}
\frac{1}{m!\cdot(m+1)!}\sum_{\pi^{-1}(0) = \lambda^{-1}(0)}\prod_{i=0}^{m}\pdf{\marm_{\pi(i)}}{\obs_i}\text{.}
\end{equation}

Note that conditioning on $\lambda^{-1}(0) = \pi^{-1}(0)$, $\pi$ is still a uniformly random permutation on $\{0,1,2,\ldots,m\}$. Therefore the two likelihoods in \eqref{eq:likeli1} and \eqref{eq:likeli2} are equal. This finishes the proof of the equivalence.
\end{proof}

\section{Missing Proofs in Section 4}
\subsection{Building Blocks}
\subsubsection{PAC algorithm for \bestkarms{}}
On an instance of \bestkarms{} with $n$ arms, the \PACSP{} algorithm in \cite{chen2016pure} is guaranteed to return a $\eps$-optimal answer of \bestkarms{} with probability $1 - \delta$, using
$$O(n\eps^{-2}(\ln\delta^{-1} + \ln k))$$
samples. Here a subset of $k$ arms $T\subseteq\inst$ is called $\eps$-optimal, if after adding $\eps$ to the mean of each arm in $T$, $T$ becomes the best $k$ arms in $\inst$.

We implement our $\PAC(S,k,\eps,\delta)$ subroutine as follows. Recall that $\PAC$ is expected to return a partition $(\sbig,\ssma)$ of the arm set $S$. If $k \le |S|/2$, we directly run \PACSP{} on the \bestkarms{} instance $S$ and return its output as $\sbig$. We let $\ssma = S\setminus\sbig$. Otherwise, we negate the mean of all arms in $S$ and run \PACSP{} to find the top $|S| - k$ arms in the negated instance.\footnote{More precisely, when the algorithm requires a sample from an arm, we draw a sample and return the opposite.} Finally, we return the output of \PACSP{} as $\ssma$ and let $\sbig = S\setminus\ssma$. In the following we prove Lemma \ref{fact:PAC}.

\begin{proof}[Proof of Lemma \ref{fact:PAC}]
By construction, the algorithm $\PAC(S,k,\eps,\delta)$ takes $$O(|S|\eps^{-2}[\ln\delta^{-1} + \ln\min(k,|S|-k)])$$ samples. In the following we prove that if $k\le |S| / 2$, the set $T$ returned by \PACSP{} is $\eps$-optimal with probability $1 - \delta$. The case $k > |S|/2$ can be proved by an analogous argument.

Let $S'$ denote the instance in which the mean of every arm in $T$ is increased by $\eps$. By definition of $\eps$-optimality, $T$ contains the best $k$ arms in $S'$. Note that the $k$-th largest mean is $S'$ is at least $\Mean{k}$. Thus for each arm $\arm\in T$, $\mu_{\arm}$ must be at least $\Mean{k} - \eps$, since otherwise even after $\mu_{\arm}$ increases by $\eps$, $\arm$ is still not among the best $k$ arms.

It also holds that every arm in $S\setminus T$ must have a mean smaller than or equal to $\Mean{k+1} + \eps$. Suppose for a contradiction that $\arm\in S\setminus T$ has a mean $\mu_{\arm} > \Mean{k+1} + \eps$. Then every arm with mean less than or equal to $\Mean{k+1}$ in $S$ still have a mean smaller than $\mu_{\arm}$ in $S'$. This implies that $\arm$ is among the best $k$ arms in $S'$, which contradicts our assumption that $\arm\notin T$.
\end{proof}

\subsubsection{PAC algorithms for \bestarm{}}
By symmetry, it suffices to implement the subroutine $\MEbig$ and prove its property. In order to estimate the mean of the largest arm in $S$, we first call $\PAC(S, 1, \eps/2, \delta/2)$ to find an approximately largest arm. Then we sample the arm $2\eps^{-2}\ln(4/\delta)$ times, and finally return its empirical mean. We prove Lemma \ref{fact:ME} as follows.

\begin{proof}[Proof of Lemma \ref{fact:ME}]
Let $\arm^*$ denote the largest arm in $S$, and let $\arm_0$ denote the arm returned by $\PAC(S,1,\eps/2,\delta/2)$. According to Lemma \ref{fact:PAC}, with probability $1 - \delta/2$, $\mu_{\arm_0}\in[\mu_{\arm^*} - \eps/2,\mu_{\arm^*}]$.
It follows that, with probability $1 - \delta/2$,
$$\left|\mu_{\arm_0} - \max_{\arm\in S}\mu_{\arm}\right|\le\eps/2\text{.}$$
Let $\hat\mu$ denote the empirical mean of arm $\arm_0$. By a Chernoff bound, with probability $1 - \delta/2$,
$$\left|\hat\mu - \mu_{\arm_0}\right|\le\eps/2\text{.}$$
It follows from a union bound that with probability $1 - \delta$,
$$\left|\hat\mu - \max_{\arm\in S}\mu_{\arm}\right|\le\eps\text{.}$$

Finally, we note that $\PAC$ consumes $O(|S|\eps^{-2}\ln\delta^{-1})$ samples as $k=1$, while sampling $\arm_0$ takes $O(\eps^{-2}\ln\delta^{-1})$ samples. This finishes the proof.
\end{proof}

\subsubsection{Elimination procedures}
We use the \textsf{Elimination} procedure defined in \cite{chen2015optimal} as our subroutine $\ELsma(S,\thetasma,\thetabig,\delta)$. The other building block $\ELbig(S,\thetasma,\thetabig,\delta)$ can be implemented either using a procedure symmetric to \textsf{Elimination}, or simply by running $\ELsma(S', -\thetabig, -\thetasma, \delta)$, where $S'$ is obtained from $S$ by negating the arms. In the following, we prove Lemma \ref{fact:EL}.

\begin{proof}[Proof of Lemma \ref{fact:EL}]
Let $T$ denote the set of arms returned by $\ELsma(S,\thetasma,\thetabig,\delta)$. Lemma B.4 in \cite{chen2015optimal} guarantees that with probability $1 - \delta$, the following three properties are satisfied:
(1) $\ELsma$ takes $O(|S|\eps^{-2}\ln\delta^{-1})$ samples, where $\eps = \thetabig - \thetasma$; (2) $$\left|\{\arm\in T:\mu_{\arm}\le\thetasma\}\right|\le|T|/10\text{;}$$
(3) Let $\arm^*$ be the largest arm in $S$. If $\mu_{\arm^*}\ge\thetabig$, then $\arm^*\in T$.

In fact, the proof of Lemma B.4 does not rely on the fact that $\arm^*$ is the largest arm in $S$. Thus property (3) holds for any fixed arm in $S$. This proves the properties of $\ELsma$. The properties of $\ELbig$ hold due to the symmetry.
\end{proof}

\subsection{Observations}
\subsubsection{Proof of Observation \ref{obs:theta}}
\begin{proof}[Proof of Observation \ref{obs:theta}]
Let $\arm$ denote the arm with the largest mean in $\ssma_r$. Recall that $\musma_r$ denote the mean of the $(\kbig_r+1)$-th largest arm in $S_r$. The correctness of $\PAC$ and Lemma \ref{fact:PAC} guarantee that $\mu_{\arm}\le\musma_r + \eps_r/8$. Note that $\musma_r$ is the $\ksma_r$-th smallest mean in $S_r$, while $\mu_{\arm}$ is the largest mean among the $\ksma_r$ arms in $\ssma_r\subseteq S_r$. So it also holds that $\mu_{\arm}\ge\musma_r$.
Thus we have
$$\mu_{\arm}\in[\musma_r,\musma_r+\eps_r/8]\text{.}$$
Moreover, as $\MEbig$ returns correctly conditioning on $\goodevent_r$, by Lemma \ref{fact:ME} we have
$$\thetabig_r\in[\musma_r - \eps_r/8,\musma_r + \eps_r/4]\text{.}$$
The second property follows from a symmetric argument.
\end{proof}

\subsubsection{Proof of Observation \ref{obs:kbound}}
\begin{proof}[Proof of Observation \ref{obs:kbound}]
Recall that $\validevent$ denotes the event that the execution of \algo{} is valid. We condition on $\validevent$ in the following proof. In particular, conditioning on $\validevent$, $\goodevent_{r-1}$ happens and $T_{r-1}$ along with the best $\kbig_{r-1}$ arms in $S_{r-1}$ constitute the correct answer of the original instance.

Let $\mubig_{r-1}$ and $\musma_{r-1}$ be the $\kbig_{r-1}$-th and the $(\kbig_{r-1}+1)$-th largest mean in $S_{r-1}$. As the arm with mean $\mubig_{r-1}$ is among the correct answer, we have $\mubig_{r-1} \ge \Mean{k}$, where $\Mean{k}$ is the $k$-th largest mean in the original instance. We also have $\musma_{r-1} \le \Mean{k+1}$ for the same reason.

Since $\goodevent_{r-1}$ happens, by Observation \ref{obs:theta} we have
$$\thetabig_{r-1}\le\musma_{r-1} + \eps_{r-1}/4\le \Mean{k+1} + \eps_{r-1} / 4\text{.}$$
Then the larger threshold used in $\ELbig$ is upper bounded by
$$\thetabig_{r-1} + \eps_{r-1} / 4 \le \Mean{k+1} + \eps_{r-1} / 2 = \Mean{k+1} + \eps_r\text{.}$$

Let $T$ denote the set of arms returned $\ELbig$ in round $r-1$. We partition $T$ into the following three parts:
$$T^{(1)} = \left\{\arm\in T:\mu_{\arm}> \Mean{k+1} + \eps_r\right\}\text{,}$$
$$T^{(2)} = \left\{\arm\in T:\Mean{k}\le \mu_{\arm}\le \Mean{k+1} + \eps_r\right\}\text{,}$$
$$T^{(3)} = \left\{\arm\in T:\mu_{\arm}\le \Mean{k+1}\right\}\text{.}$$

By Lemma \ref{fact:EL} and the correctness of $\EL$ conditioning on $\goodevent_{r-1}$, we have
$$|T^{(1)}|\le |T|/10\text{.}$$
It follows that $$|T^{(2)}| + |T^{(3)}| \ge 9|T|/10\ge |T|/2\text{.}$$

By definition of arm groups, every arm in $T^{(2)}$ is in $\gbig_{\ge r}$. In order to bound $T^{(3)}$, we say that an arm is misclassified into $\sbig_{r-1}$, if the arm is not among the best $\kbig_{r-1}$ arms in $S_{r-1}$, but is included in $\sbig_{r-1}$. We may define misclassification into $\ssma_{r-1}$ similarly. As $|\sbig_{r-1}| = \kbig_{r-1}$, the numbers of arms misclassified into both sides are the same.

Since the arms in $T^{(3)}$ are misclassified into $\sbig_{r-1}$, there are at least $|T^{(3)}|$ other arms misclassified into $\ssma_{r-1}$. Lemma \ref{fact:PAC} (along with the correctness of $\PAC$) guarantees that all arms misclassified into $\ssma_{r-1}$ have means smaller than or equal to $\Mean{k+1} + \eps_{r-1}/8$. Thus by definition of arm groups, all these $|T^{(3)}|$ arms are also in $\gbig_{\ge r}$. Therefore, we have
$$|\gbig_{\ge r}|\ge|T^{(2)}| + |T^{(3)}|\ge |T|/2\text{.}$$

Note that $|T| = \kbig_r$. Therefore we conclude that $\kbig_r\le 2|\gbig_{\ge r}|$. The bound on $\ksma_r$ can be proved using a symmetric argument.
\end{proof}

\subsection{Proof of Lemma \ref{lem:valid}}
\textbf{Lemma~\ref{lem:valid}}
(restated)
{\em	
$\Pr\left[\validevent\right]\ge 1 - \delta$.
}
\begin{proof}[Proof of Lemma \ref{lem:valid}]
We prove the lemma by upper bounding the probability of $\overline{\validevent}$, the complement of $\validevent$.

\textbf{Split $\overline{\validevent}$.} Let $\badevent_r$ denote the event that \algo{} is valid at round $r$, yet it becomes invalid at round $r+1$. Then we have
$$\Pr\left[\overline{\validevent}\right] = \sum_{r = 1}^{\infty}\Pr\left[\badevent_r\right]\text{.}$$

By definition of validity, event $\badevent_r$ happens in one of the following two cases:
\begin{itemize}
\item Case 1: $\goodevent_r$ does not happen.
\item Case 2: $\goodevent_r$ happens, yet $T_{r+1}$ together with the best $\kbig_{r+1}$ arms in $S_{r+1}$ is no longer the correct answer.
\end{itemize}

The probability of Case 1 is upper bounded by $5\delta_r$ according to Observation \ref{obs:event}. We focus on bounding the probability of Case 2 in the following.

\textbf{Misclassified arms.} Recall that $\mubig_r$ and $\musma_r$ denote the means of the $\kbig_r$-th and the $(\kbig_r+1)$-th largest arms in $S_r$ respectively. Conditioning on the validity of the execution at round $r$, the arm with mean $\mubig_r$ is among the best $k$ arms in the original instance, while the arm with mean $\musma_r$ is not. Thus we have
$$\mubig_r\ge\Mean{k}>\Mean{k+1}\ge\musma_r\text{.}$$
Define
$$\ubig_r = \{\arm\in\sbig_r:\mu_{\arm}\le\musma_r\}$$
and
$$\usma_r = \{\arm\in\ssma_r:\mu_{\arm}\ge\mubig_r\}\text{.}$$
In other words, $\ubig_r$ and $\usma_r$ denote the set of arms ``misclassified'' by the $\PAC$ subroutine into $\sbig_r$ and $\ssma_r$ in round $r$.

\textbf{Bound the number of misclassified arms.} Note that since $|\ubig_r|\le|\sbig_r|=\kbig_r$, and in addition, less than $\ksma_r$ arms in $S_r$ have means smaller than or equal to $\musma_r$,
$$|\ubig_r|\le\min(\kbig_r, \ksma_r)\text{.}$$
For the same reason, it holds that
$$|\usma_r|\le\min(\kbig_r, \ksma_r)\text{.}$$

\textbf{With high probability, no misclassified arms are removed.} By Observation \ref{obs:theta}, conditioning on $\goodevent_r$, we have
$$\thetabig_r \ge \musma_r - \eps_r/8\text{.}$$
Therefore, when $\ELbig$ in Line \ref{line:EL} is called at round $r$, the smaller threshold is at least
$$\thetabig_r + \eps_r/8 \ge \musma_r\text{,}$$
which is larger than the mean of every arm in $\ubig_r$. By Lemma \ref{fact:EL} and a union bound, with probability
$$1 - |\ubig_r|\deltap_r \ge 1 - \min(\kbig_r,\ksma_r)\deltap_r = 1 - \delta_r\text{,}$$
no arms in $\ubig$ are removed by $\ELbig$.
For the same reason, with probability $1 - \delta_r$, no arms in $\usma$ are removed by $\ELsma$.

\textbf{Bound the probability of Case 2.} Thus, with probability at least $1 - 2\delta_r$ conditioning on $\goodevent_r$, $\ELbig$ only removes arms with means larger than or equal to $\mubig_r$, and $\ELsma$ only removes arms with means smaller than or equal to $\musma_r$. Consequently, every arm in $S_r$ with mean greater than or equal to $\mubig_r$ either moves to $T_{r+1}$ or stays in $S_{r+1}$, which implies that Case 2 does not happen.

Therefore, the Case 2 happens with probability at most $2\delta_r$, and it follows that
$$\Pr\left[\badevent_r\right]\le 5\delta_r + 2\delta_r = 7\delta_r\text{.}$$
Finally, we have
$$\Pr\left[\overline{\validevent}\right]\le \sum_{r = 1}^{\infty}7\delta_r\le \sum_{r=1}^{\infty}\frac{7\delta}{20r^2}\ge \delta\text{.}$$
\end{proof}

\subsection{Missing Calculation in the Proof of Lemma \ref{lem:sample}}
\textbf{Lemma~\ref{lem:sample}}
(restated)
{\em	
Conditioning on event $\validevent$, \algo{} takes $O(H\ln\delta^{-1} + \htilbig + \htilsma + \htil)$ samples.
}
\begin{proof}[Proof (continued)]
Recall that
$$\Hone_r = \left(|\gbig_{\ge r}|+|\gsma_{\ge r}|\right)\eps_r^{-2}(\ln\delta^{-1} + \ln r)\text{,}$$
$$\Htwobig_r = \eps_r^{-2}|\gbig_{\ge r}|\ln|\gsma_{\ge r}|\text{,}$$
$$\Htwosma_r = \eps_r^{-2}|\gsma_{\ge r}|\ln|\gbig_{\ge r}|\text{.}$$
Our goal is to show that
$$\sum_{r=1}^{\infty}\Hone_r = O\left(H\ln\delta^{-1} + \htil\right)\text{,}$$
$$\sum_{r=1}^{\infty}\Htwobig_r = O\left(\htilbig\right)\text{, and}$$
$$\sum_{r=1}^{\infty}\Htwosma_r = O\left(\htilsma\right)\text{.}$$

\textbf{Upper bound the $\Hone$ term:} It follows from a directly calculation that
\begin{equation*}\begin{split}
\sum_{r=1}^{\infty}\Hone_r
=&\sum_{r=1}^{\infty}\sum_{i = r}^{\infty}\left(|\gbig_{i}|+|\gsma_{i}|\right)\eps_r^{-2}(\ln\delta^{-1} + \ln r)\\
=&\sum_{i=1}^{\infty}\left(|\gbig_{i}|+|\gsma_{i}|\right)\sum_{r=1}^{i}\eps_r^{-2}(\ln\delta^{-1} + \ln r)\\
=&O\left(\sum_{i=1}^{\infty}\left(|\gbig_{i}|+|\gsma_{i}|\right)\eps_i^{-2}(\ln\delta^{-1} + \ln i)\right)\\
=&O\left(\sum_{i=1}^{n}\Gap{i}^{-2}\left(\ln\delta^{-1} + \ln\ln\Gap{i}^{-1}\right)\right)\text{.}
\end{split}\end{equation*}
Here the second step interchanges the order of summation. The third step holds since the inner summation is always dominated by the last term. Finally, the last step is due to the fact that $\Delta_{\arm} = \Theta(\eps_i)$ for every arm $\arm\in\gbig_i\cup\gsma_i$.
Therefore we have
$$\sum_{r=1}^{\infty}\Hone_r = O(H\ln\delta^{-1} + \htil)\text{.}$$

\textbf{Upper bound $\Htwobig$ and $\Htwosma$:}
By definition of $\Htwobig_r$, we have
\begin{equation*}\begin{split}
\sum_{r=1}^{\infty}\Htwobig_r
=&\sum_{r=1}^{\infty}\sum_{i = r}^{\infty}\eps_r^{-2}|\gbig_{i}|\ln|\gsma_{\ge r}|\\
=&\sum_{i=1}^{\infty}|\gbig_{i}|\sum_{r=1}^{i}\eps_r^{-2}\ln|\gsma_{\ge r}|\text{.}
\end{split}\end{equation*}

Therefore we conclude that $$\sum_{r=1}^{\infty}\Htwobig_r = O(\htilbig)\text{.}$$ The bound on the sum of $\Htwosma_r$ follows from an analogous calculation.
\end{proof}

\subsection{Proof of Theorem \ref{theo:ubrefine}}
\textbf{Theorem~\ref{theo:ubrefine}}
(restated)
{\em	
For every \bestkarms{} instance,
 the following statements hold:
\begin{enumerate}
\item
$\htilbig + \htilsma = O\left(\left(\hbig+\hsma\right)\ln\ln n\right)\text{.}$
\item
$\htilbig + \htilsma = O\left(H\ln k\right)\text{.}$
\end{enumerate}
}
\begin{proof}[Proof of Theorem \ref{theo:ubrefine}]
\textbf{First Upper Bound.} Recall that $$\hbig = \sum_{i=1}^{\infty}\left|\gbig_i\right|\cdot\max_{j\le i}\eps_j^{-2}\ln\left|\gsma_{\ge j}\right|\text{, and}$$
$$\htilbig = \sum_{i=1}^{\infty}\left|\gbig_i\right|\sum_{j=1}^{i}\eps_j^{-2}\ln\left|\gsma_{\ge j}\right|\text{.}$$

For brevity, let $N_r$ denote $\eps_r^{-2}\ln|\gsma_{\ge r}| = 4^r\ln|\gsma_{\ge r}|$. We fix the value $i$. Then the $i$-th term in $\htilbig$ reduces to $\left|\gbig_i\right|\sum_{r=1}^{i}N_r\text{.}$ Let $r^* = \argmax_{1\le r\le i}N_r$. Thus the $i$-th term in $\hbig$ is simply $\left|\gbig_{i}\right|N_{r^*}$, which is in general smaller than $\left|\gbig_i\right|\sum_{r=1}^{i}N_r$. However, we will show that the ratio between the two terms is bounded by $O(\ln\ln n)$.

By definition of $r^*$, we have $N_{r^*}\ge N_i$. Substituting $N_{r^*}$ and $N_i$ yields
$$4^{r^*}\ln\left|\gsma_{\ge r^*}\right|\ge 4^{i}\ln\left|\gsma_{\ge i}\right|\text{.}$$
It follows that
$$4^{i-r^*}\ln\left|\gsma_{\ge i}\right|\le\ln\left|\gsma_{\ge r^*}\right|\le\ln n\text{,}$$
and thus $i - r^* = O(\ln\ln n)$.

Let $1\le r_1\le r^*$ be the smallest integer such that $N_{r_1}\ge 2^{r_1 - r^*}N_{r^*}$. By substituting $N_{r_1}$ and $N_{r^*}$, we obtain
$$4^{r_1}\ln\left|\gsma_{\ge r_1}\right|\ge 2^{r_1 - r^*}\cdot4^{r^*}\ln\left|\gsma_{\ge r^*}\right|\text{,}$$
which further implies that
$$2^{r^*-r_1}\ln|\gsma_{\ge r^*}|\le \ln|\gsma_{\ge r_1}|\le\ln n$$
and thus $r^* - r_1 = O(\ln\ln n)$.

Therefore we have $i - r_1 = O(\ln\ln n)$, and we can bound the sum of $N_r$ as follows:
\begin{equation*}\begin{split}
\sum_{r=1}^{i}N_r
=&\sum_{r=1}^{r_1 - 1}N_r + \sum_{r=r_1}^{i}N_r\\
\le&N_{r^*}\sum_{r=1}^{r_1-1}2^{r-r^*} + (i-r_1+1)N_{r^*}\\
\le &(i-r_1+2)N_{r^*} = O(N_{r^*}\ln\ln n)\text{.}
\end{split}\end{equation*}
Here the second step follows from $N_{r} < 2^{r - r^*}N_{r^*}$ for $r < r_1$ (by definition of $r_1$) and $N_r \le N_{r^*}$ for $r\ge r_1$ (by definition of $r^*$).

It then follows from a direct summation over all $i$ that
$$\htilbig = O(\hbig\ln\ln n)\text{.}$$
The bound on $\htilsma$ can be proved similarly.

\textbf{Second Upper Bound.} Note that
\begin{equation}\begin{split}\label{eq:htilbig}
\htilbig &= \sum_{i=1}^{\infty}\left|\gbig_i\right|\sum_{j=1}^{i}\eps_j^{-2}\ln\left|\gsma_{\ge j}\right|\\
&= \sum_{j=1}^{\infty}\eps_j^{-2}\ln\left|\gsma_{\ge j}\right|\sum_{i=j}^{\infty}\left|\gbig_i\right|\\
&= \sum_{i=1}^{\infty}\eps_i^{-2}\left|\gbig_{\ge i}\right|\ln\left|\gsma_{\ge i}\right|\text{.}
\end{split}\end{equation}
Here the second step interchanges the order of summation. By symmetry we also have
\begin{equation}\label{eq:htilsma}
\htilsma = \sum_{i=1}^{\infty}\eps_i^{-2}\left|\gsma_{\ge i}\right|\ln\left|\gbig_{\ge i}\right|\text{.}
\end{equation}

It can be easily verified that for $1\le x\le y$, we have
\begin{equation}\label{eq:tech}
x\ln y + y\ln x \le (x+y)(2\ln x + 1)\text{.}
\end{equation}
Note that $\min\left(|\gbig_{\ge i}|,|\gsma_{\ge i}|\right) \le k$ for all $i$. Therefore we can bound $\htilbig + \htilsma$ as follows:
\begin{equation*}\begin{split}
&\htilbig + \htilsma\\
=& \sum_{i=1}^{\infty}\eps_i^{-2}\left(|\gbig_{\ge i}|\ln|\gsma_{\ge i}| + |\gsma_{\ge i}|\ln|\gbig_{\ge i}|\right)\\
=& O\left(\sum_{i=1}^{\infty}\eps_i^{-2}\left(|\gbig_{\ge i}| + |\gsma_{\ge i}|\right)\ln\min\left(|\gbig_{\ge i}|,|\gsma_{\ge i}|\right)\right)\\
=& O\left(\sum_{i=1}^{\infty}\eps_i^{-2}\left(|\gbig_{\ge i}| + |\gsma_{\ge i}|\right)\ln k\right)=O(H\ln k)\text{.}
\end{split}\end{equation*}
The first step follows from \eqref{eq:htilbig} and \eqref{eq:htilsma}. The second step is due to \eqref{eq:tech}. The third step is due to the observation that $\min\left(|\gbig_{\ge i}|,|\gsma_{\ge i}|\right) \le k$. Finally, the last step follows from a simple rearrangement of the summation:
\begin{equation*}\begin{split}
&\sum_{i=1}^{\infty}\eps_i^{-2}\left(|\gbig_{\ge i}| + |\gsma_{\ge i}|\right)\\
=&\sum_{i=1}^{\infty}\eps_i^{-2}\sum_{j=i}^{\infty}\left(|\gbig_{j}| + |\gsma_{j}|\right)\\
=&\sum_{j=1}^{\infty}\left(|\gbig_{j}| + |\gsma_{j}|\right)\sum_{i=1}^{j}\eps_i^{-2}\\
=&O\left(\sum_{j=1}^{\infty}\eps_j^{-2}\left(|\gbig_{j}| + |\gsma_{j}|\right)\right) = O(H)\text{.}
\end{split}\end{equation*}

\end{proof}

\end{document}